\documentclass[11pt,a4paper]{article}
\usepackage[round]{natbib}
\usepackage{fullpage}
\usepackage{cmap}
\usepackage{algorithm,algorithmic}

\usepackage{microtype}
\usepackage{graphicx}
\usepackage{subfigure}
\usepackage{booktabs}
\usepackage[multiple]{footmisc}

\usepackage{amsmath,amsfonts,amssymb,amsthm}
\usepackage{mathrsfs}

\usepackage[dvipsnames]{xcolor}
\usepackage{url}
\usepackage[colorlinks]{hyperref}
\hypersetup{
    linkcolor=red,
    citecolor=OliveGreen,
    filecolor=black,
    urlcolor=black,
}

\usepackage{algorithm}

\usepackage{pdfpages}
\usepackage{comment}
\usepackage{multirow}
\usepackage{dsfont,makecell}
\usepackage{color, colortbl}

\DeclareMathOperator*{\Argmin}{Argmin}

\DeclareMathOperator*{\argmin}{argmin}
\DeclareMathOperator*{\argmax}{argmax}

\newtheorem{theorem}{Theorem}[section]

\newtheorem{proposition}{Proposition}[section]
\newtheorem{lemma}{Lemma}[section]
\newtheorem{ass}{Assumption}
\newtheorem{rem}{Remark}[section]
\newtheorem{corollary}{Corollary}[section]

\newcommand{\ba}{\begin{array}{c}}
\newcommand{\bal}{\begin{array}{l}}
\newcommand{\ea}{\end{array}}

\newcommand{\R}{\mathds{R}}

\newcommand{\bit}{\begin{itemize}}
\newcommand{\eit}{\end{itemize}}

\newcommand{\bvec}{\left(\!\!\!\begin{array}{c} }
\newcommand{\evec}{\end{array}\!\!\!\right)}
\newcommand{\E}{\mathds{E}}

\newcommand{\tr}{\mbox{tr}}

\newcommand{\wt}{\widetilde}

\newcommand{\rb}{\mathbb{R}}

\newcommand{\VV}{\mathcal{V}}

\def \FF{{\mathcal F}}

\newcommand{\noteb}[1]{{\textbf{\color{blue}#1}}}
\newcommand{\puone}{p_U^1}
\newcommand{\putwo}{p_U^2}
\newcommand{\quone}{q_U^1}
\newcommand{\qutwo}{q_U^2}
\newcommand{\pvone}{p_V^1}
\newcommand{\pvtwo}{p_V^2}
\newcommand{\qvone}{q_V^1}
\newcommand{\qvtwo}{q_V^2}

\newcommand{\sgn}{\textup{sign}}

\newcommand{\Unorm}{\mathscr{U}}
\newcommand{\Vnorm}{\mathscr{V}}
\newcommand{\Wnorm}{\mathscr{W}}

\newcommand{\F}{\textup{F}}

\newcommand{\ff}{\mathfrak{f}}

\newcommand{\UU}{\mathcal{U}}

\newcommand{\barr}{\begin{array}}
\newcommand{\earr}{\end{array}}

\newcommand{\ind}{\mathds{1}}
\newcommand{\zeros}{{0}}
\newcommand{\cV}{\mathcal{V}}
\newcommand{\WW}{\mathcal{W}}

\renewcommand{\tilde}{\widetilde}

\newcommand{\leqs}{\leqslant}
\newcommand{\geqs}{\geqslant}
\newcommand{\cL}{\mathcal{L}}
\newcommand{\wh}{\widehat}
\newcommand{\veps}{\varepsilon}
\newcommand{\cH}{\mathcal{H}}
\newcommand{\prim}{\textup{prim}}
\newcommand{\dual}{\textup{dual}}
\newcommand{\lang}{\left\langle}
\newcommand{\rang}{\right\rangle}
\newcommand{\Rem}{\textup{\texttt{r}}}
\newcommand{\Edata}{\mathds{E}_{\textup{data}}}

\newcommand{\Gap}{\textup{\texttt{Gap}}}
\renewcommand{\tr}{\textup{tr}}

\newcommand{\Acur}{A}
\newcommand{\Bcur}{B}
\newcommand{\Aprev}{A_{\textup{pr}}}
\newcommand{\Bprev}{B_{\textup{pr}}}

\newcommand{\Usum}{U_{\Sigma}}

\newcommand{\Vsum}{V_{\Sigma}}

\renewcommand{\it}{{\boldsymbol{i}}}
\newcommand{\jt}{{\boldsymbol{j}}}
\newcommand{\lt}{{\boldsymbol{l}}}
\newcommand{\ellt}{{\boldsymbol{\ell}}}

\newcommand{\List}{\texttt{Changed}}

\newcommand{\resetProcCounter}{
\floatname{algorithm}{Procedure}
\setcounter{algorithm}{0}
}

\renewcommand{\cite}{\citep}

\title{Efficient Primal-Dual Algorithms \\ for Large-Scale Multiclass Classification}

\newcommand*{\affmark}[1][*]{\textsuperscript{#1}}
\newcommand*{\email}[1]{\texttt{#1}}
\author{
Dmitry Babichev \thanks{Equal contribution.} \textsuperscript{,}\thanks{SIERRA Project-Team, INRIA and \' Ecole Normale Sup\' erieure, PSL Research University, Paris, France.}
\and 
Dmitrii M.~Ostrovskii \affmark[$*$,]\affmark[$\dagger$] 
\and 
Francis Bach \affmark[$\dagger$]
}
\date{
\email{\{dmitry.babichev, dmitrii.ostrovskii, francis.bach\}@inria.fr}
}

\begin{document}
\maketitle

\begin{abstract}
We develop efficient algorithms to train~$\ell_1$-regularized linear classifiers with large dimensionality~$d$ of the feature space, number of classes~$k$, and sample size~$n$.
Our focus is on a special class of losses that includes, in particular, the multiclass hinge and logistic losses.
Our approach combines several ideas:
(i) passing to the equivalent saddle-point problem with a quasi-bilinear objective; 
(ii) applying stochastic mirror descent with a proper choice of geometry which guarantees a favorable accuracy bound;
(iii) devising non-uniform sampling schemes to approximate the matrix products.
In particular, for the multiclass hinge loss we propose a \emph{sublinear} algorithm with iterations performed in~$O(d+n+k)$ arithmetic operations.
\end{abstract}
\section{Introduction}
\label{sec:intro}
We study optimization problems arising in multiclass linear classification with a large number of classes and features.
Formally, consider a dataset of~$n$ pairs~$(x_i, y_i)$,~$i \in [n] := \{1, ..., n\}$, where~$x_i \in \rb^d$ is the feature vector of the~$i$-th example, and~$y_i \in \{e_1, ..., e_k\}$ is the label vector encoding one of~$k$ possible classes; here~$e_1, ..., e_k$ are the standard basis vectors in~$\R^k$.
Given such data, our goal is to find a linear classifier that minimizes the $\ell_1$-regularized empirical risk. We thus consider a  minimization problem of the form
\begin{equation}
\label{eq:erm}
\min\limits_{U\in\R^{d\times k}}  \frac{1}{n} \sum_{i=1}^n \ell(U^\top x_i,y_i) + \lambda\|U\|_{1}.
\end{equation}
Here,~$U \in \R^{d \times k}$ is the matrix whose columns specify the parameter vectors for each of the~$k$ classes;
~$\ell(U^\top x,y)$, with~$\ell: \R^k \times \Delta_k \to \R$ and~$\Delta_k \subset \R^k$ being the unit probability simplex, is the \textit{loss} corresponding to the margins~$U^\top x \in \rb^k$ assigned to~$x$ when its class is encoded by~$y$; finally, the regularization term~$\lambda\|U\|_{1}$,~$\lambda \geqs 0$, uses the elementwise~$\ell_1$-norm~$\|U\|_{1} = \sum_{i=1}^d \sum_{j=1}^k |U_{ij}|$. 
Apart from inducing sparsity of features and classes~\citep{buhlmann2011statistics}, this choice of regularization is crucial from the algorithmic perspective, as will be explained in Sec.~\ref{sec:sampling}. 

Our focus is on the so-called \textit{Fenchel-Young} losses, introduced by~\citet{blondel2018learning}, which can be expressed as
\begin{equation}
\label{eq:fenchel-loss}
\ell(U^\top x,y) = \max\limits_{v\in\Delta_k} \left\{- \ff(v,y) + (v-y)^\top U^\top x \right\},
\end{equation}
where~$\Delta_k$ is the probability simplex in~$\R^k$, and the function~$\ff(\cdot,y): \Delta_k \to \R$ is convex and ``simple'' (i.e., quasi-separable in~$v$), which implies that maximization in~\eqref{eq:fenchel-loss} can be performed in running time~$O(k)$.
In particular, this allows us to address two commonly used multiclass losses:

\begin{itemize}
\item The multiclass \textit{logistic (or softmax) loss} 
\begin{equation}
\label{eq:loss-logistic}
\log \left(\sum_{l=1}^k \exp(U_l^\top x)\right) - y^\top U^\top x,
\end{equation}
where~$U_l$ is the~$l$-th column of~$U$ so that~$U_l^\top x$ is the~$l$-th element of~$U^\top x$. This loss corresponds to~\eqref{eq:fenchel-loss} with the negative entropy term~$\ff(v,y) = \sum_{l=1}^k v_l\log v_l$, which is  independent of~$y$.
\item The multiclass \textit{hinge loss}, given by
\begin{equation}
\label{eq:loss-hinge}
\max_{l \in [k]} \left\{ \ind[e_l \ne y] + U_{l}^\top  x \right\}  - y^\top U^\top x,
\end{equation}
and used in multiclass support vector machines (SVM). This loss reduces to~\eqref{eq:fenchel-loss} by setting~$\ff(v,y) = v^\top y - 1$ (refer to Appendix~\ref{SVMMotivation} for additional details).
\end{itemize}

Arranging the feature vectors into~$X \in \rb^{n \times d}$, and the class labels into~$Y \in \rb^{n \times k}$, and
using the Fenchel-type representation~\eqref{eq:fenchel-loss} of the loss, we can recast the initial problem~\eqref{eq:erm} as the following convex-concave saddle-point problem:
\begin{align}
\label{eq:erm-saddle}
\min_{ U \in \rb^{d \times k} } \max_{V \in \cV} \;
&-\FF(V,Y) +
\frac{1}{n} \tr \left[
( V - Y )^\top X U
\right]  + \lambda \| U \|_{1},\\
\label{eq:composite-part-v}
&\mbox{ with} \;\;  \FF(V,Y) := \frac{1}{n} \sum_{i=1}^n \ff(v_i,y_i),
\end{align}
where~$v_i, y_i \in \Delta_k$ are the~$i$-th rows of~$V$ and~$Y$, and define the Cartesian product of probability simplices
\begin{equation}
\label{eq:cube-of-simplices}
\cV := \Delta_k^{\otimes n} \subset \R^{n \times k},
\end{equation} 
the set comprised of all right-stochastic matrices in~$\R^{n \times k}$.
Taking into account the Fenchel-type representation~\eqref{eq:fenchel-loss}, this reduction is quite natural.
Indeed, while the objective in~\eqref{eq:erm} can be non-smooth, the essential part of the objective in~\eqref{eq:erm-saddle},
\begin{equation}
\label{eq:bilinear-part}
\Phi(U,V-Y) := \frac{1}{n} \tr \left[
( V - Y )^\top X U
\right],
\end{equation}
is not only smooth but \textit{bilinear} in~$U$ and~$V-Y$.
On the other hand, the presence of the dual constraints, as given by~\eqref{eq:cube-of-simplices}, does not seem problematic since~$\cV$ allows for a computationally cheap projection oracle.
Finally, in the saddle-point formulation we can control the duality gap which provides an \emph{accuracy certificate} for the initial problem, see, e.g.,~\citet{nemirovski2010accuracy,algorec2018arxiv}.

In this work, we propose efficient algorithms for solving~\eqref{eq:erm} via the associated saddle-point problem~\eqref{eq:erm-saddle}, built upon the vector-field formulation of stochastic mirror descent (SMD), a well-known general optimization method, see, e.g.,~\citet{optbook1} and references therein. 
Mirror descent, as well as its cousin Mirror Prox~\cite{optbook2}, allows to solve convex-concave saddle-point problems (CCSPPs) whenever the first-order information about the objective is available, and the primal and dual feasible sets have simple structures that allow for easily computable prox-mappings.
While these algorithms are poorly adapted for obtaining high-accuracy solutions, this is not a limitation in the context of empirical risk minimization where the ultimate goal is to minimize the \textit{expected risk}, and there is often a natural level of statistical accuracy, going beyond which is unnecessary, see, e.g.,~\citet{mokhtari2016adaptive}.
On the other hand, mirror descent is especially well-suited to quasi-bilinear CCSPPs of the form~\eqref{eq:erm-saddle}. 

First, it uses the Bregman divergence, rather than the standard Euclidean distance, as a proximity measure, which allows to adjust to the specific geometry associated to~$\cV$ and~$\|\cdot\|_{1}$.

Second, it retains its favorable convergence guarantee when the exact partial gradients
\[
\nabla_{U}[\Phi(U,V-Y)] = \frac{1}{n}XU, \quad \nabla_{V}[\Phi(U,V-Y)] = \frac{1}{n} X^\top(V-Y),
\] 
cf.~\eqref{eq:bilinear-part}, are replaced with their unbiased estimates. 
This is especially important: recall that the exact computation of the matrix products~$XU,~X^\top (V-Y)$ requires~$O(dnk)$ arithmetic operations (a.o.'s) while the subsequent proximal mappings can usually be done in linear time in the combined size of the variables, i.e.,~$O(dk+nk)$; thus, computation of the matrix products becomes the main bottleneck.
On the other hand, these matrix products can be approximated via \textit{randomized  subsampling} of the elements of~$U$,~$V-Y$, and~$X$.
While a similar approach has already been explored by~\citet{optbook2} in the case of bilinear CCSPPs with vector variables arising in sparse recovery, its extension to problems of the type~\eqref{eq:erm-saddle} is non-trivial.
In fact, for a sampling scheme to be deemed ``good'', it clearly has to satisfy two concurrent requirements.

\begin{itemize}
\item[(a)]
On one hand, one must control the stochastic variability of the estimates in the chosen sampling scheme. 
Ideally, the additional term due to sampling should not be much larger than the term already present in the accuracy bound for deterministic mirror descent.
\item[(b)]
On the other hand, the estimates must be computationally cheap. 
The immediate goal is~$O(dk+nk)$ per iteration, i.e., the cost of the proximal mapping given the full gradients.
However, one might want to go beyond that, to~$O(d+n+k)$ per iteration, obtaining a \textit{sublinear} algorithm with complexity of an iteration much smaller than the combined size of the variables.
\end{itemize}
Devising a sampling scheme that satisfies both these requirements simultaneously is a delicate task. To solve it, one should carefully exploit the specific geometric structure of~\eqref{eq:erm-saddle} associated to~$\cV$, the norm~$\|\cdot\|_{1}$, and the term~$\FF(V,Y)$.

\paragraph{Contributions and Outline.} 
We propose two sampling schemes with various levels of ``aggressiveness'' satisfying the above requirements, study their statistical properties, and analyze the numerical complexity of stochastic mirror descent (SMD) equipped with them.
In particular, we show that SMD with appropriately balanced entropy-type potentials (see Sec.~\ref{sec:geometry-choice-main}) has nearly the same complexity estimate, in terms of the number of iterations to guarantee a given duality gap, as its deterministic counterpart with exact gradients, while at the same time enjoying a drastically improved cost of iterations, with improvement depending on the scheme.

\begin{itemize}
\item The \emph{partial sampling scheme} (see Sec.~\ref{sec:part-stoch}) is applicable to any problem of the form~\eqref{eq:erm-saddle}.
In it, we sample a single \textit{row} of~$U$ and~$V-Y$ at a time, with probabilities minimizing the expected squared norm of the gradient estimate at the current iterate. 
This leads to the cost~$O(dn + dk+nk)$ of iterations. 

\item In the \emph{full sampling scheme} (see Sec.~\ref{sec:full-stoch}), sampling of the rows of~$U$ and~$V-Y$ is augmented with column sampling. 
Applying it to the multiclass hinge loss~\eqref{eq:loss-hinge}, we construct a sublinear algorithm (see Sec.~\ref{sec:sublinear}) with iteration cost of~$O(d+n+k)$ a.o.'s, and a one-time additional cost of~$O(dn+nk)$ a.o.'s (starting from the zero primal solution allows to remove the~$O(dk)$ term).

\end{itemize}
We conclude the paper with numerical experiments that illustrate our approach (see Sec.~\ref{sec:experiments}).

\paragraph{Related Work.}
Instead of passing to the saddle-point problem~\eqref{eq:erm-saddle} we can solve the original problem~\eqref{eq:erm} directly.
For deterministic first-order algorithms, this leads to~$O(dnk)$ numerical complexity of one iteration, and for  stochastic algorithms that sample one example at a time, the complexity is~$O(dk)$, similarly to our approach with partial sampling. In particular, algorithms such as SAG~\citep{schmidt2017minimizing}, SVRG~\citep{johnson2013accelerating}, SDCA~\citep{shalev2013stochastic}, SAGA~\citep{defazio2014saga} use variance reduction techniques to obtain accelerated convergence rates in terms of the number of iterations, but all have $O(dk)$ or $O(dk+nk)$ runtime. Such variance reduction techniques have been extended to saddle-point problems~\citep{palaniappan2016stochastic,shi2017bregman} but with the same overall complexity (our full sampling scheme could also be adapted to them); still, none of these algorithms are sublinear.

Regarding sublinear algorithms, several results can be found in the literature for the biclass setting. 
The case of bilinear~CCSPPs was first considered by~\citet{grigoriadis1995sublinear}, then by~\citet{optbook2}, and by~\citet{xiao2017dscovr}. 
\citet{hazan2011beatingsvm} proposed a sublinear algorithm for biclass SVM, and~\citet{garber2011approximating,garber2016sublinear} addressed semidefinite programs; more general results were given by~\citet{clarkson2012sublinear}.
We reuse some of the tools (mirror descent and multiplicative updates) considered in this literature.
However, none of these approaches can be easily extended to the multiclass setting without an extra~$O(k)$ factor in the cost of iterations. 

\section{Mirror Descent for Quasi-Bilinear CCSPPs}
\label{sec:geometry-choice-main}

\paragraph{Preliminary Reductions.}
We focus on a CCSPP given by~\eqref{eq:erm-saddle}--\eqref{eq:bilinear-part}, and make the following assumption:
\begin{ass}
\label{ass:radius}
We assume that the~$\|\cdot\|_{1}$-radius of an optimal solution $U^\ast$ to~\eqref{eq:erm-saddle} is known:
\[
\|U^\ast\|_{1} = R_{*}.
\]
\end{ass}
\begin{rem}
\label{rem:radius}
The accuracy bounds presented later on depend on~$R_{*}$, and are preserved when~$R_{*}$ becomes an upper bound on~$\|U^\ast\|_1$.
Since~$U=0$ is feasible, we always have~$R_{*}  \leqs 1/\lambda$ when the loss is non-negative, but this bound is usually loose since~$\lambda$ typically decreases with~$n$. 
Alternatively, we can solve a series of constrained problems, starting with a small radius, and increasing it by a constant factor until the obtained solution leaves the boundary of the feasible set. 
\end{rem}
With Assumption~\ref{ass:radius}, we can put~\eqref{eq:erm-saddle} in the constrained form:
\begin{equation}
\label{MainEquationConstrained}
\min_{\| U \|_1 \le R_{*}}\max\limits_{V\in \VV} - \FF(V,Y) + \Phi(U,V-Y) + \lambda \|U\|_{1},
\end{equation}
and then reduce~\eqref{MainEquationConstrained} to a simplex-constrained problem.
Indeed, let~$\ind_{2d \times k} \in \R^{2d \times k}$ be the all-ones matrix, and define
\begin{equation}
\label{eq:solid-simplex}
\UU := \big\{U \in \R^{2d \times k}:  U_{ij} \geqs 0, \; \tr[\ind_{2d \times k}^\top U] \leqs R_{*}\big\},
\end{equation}
i.e., the ``solid'' simplex in~$\R^{2d \times k}$ (note that~$\tr[\ind_{2d \times k}^\top U] = \| U\|_{1}$ on~$\UU$).
Consider now the following CCSPP:
\begin{equation}
\label{SaddleUhat}
\min_{U \in \UU} \max_{\vphantom{U \in \UU}V\in \VV} - \FF(V,Y) + \wh \Phi(U, V-Y) + \lambda \, \tr[\ind_{2d \times k}^\top U],
\end{equation}
where~$\FF(V,Y)$ and~$\VV$ are given by~\eqref{eq:composite-part-v}--\eqref{eq:cube-of-simplices},~$\UU$ by~\eqref{eq:solid-simplex}, and
\begin{equation}
\label{eq:bilinear-part-simplex}
\wh\Phi(U,V-Y) := \frac{1}{n} \tr \left[(V - Y)^\top \wh X U\right], \; \wh X := \begin{bmatrix}X, -X \end{bmatrix},
\end{equation}
using the ``Matlab notation'' for matrix concatenation (i.e.,~$\wh X \in \R^{n \times 2d}$).
One can verify that~\eqref{SaddleUhat} is equivalent to~\eqref{MainEquationConstrained}, and hence to~\eqref{eq:erm-saddle}, in the following sense: any~$\veps$-accurate (in terms of the primal accuracy or duality gap, see below) solution~$(U,V)$ to~\eqref{SaddleUhat} with $U = [U_1; U_2]$ results in the~$\veps$-accurate solution~$(U_1 - U_2,V)$ to~\eqref{MainEquationConstrained}.
This reduction is motivated by the fact that mirror descent with an entropy-type potential on~$\UU$ reduces to multiplicative updates, which is crucial in the sublinear algorithm presented in~Sec.~\ref{sec:sublinear}.

\paragraph{Background on CCSPPs.} 
The accuracy of a candidate solution~$(\bar U, \bar V)$ to a~CCSPP
\[
\min_{U \in \UU} \max_{V \in \VV} f(U,V)
\]
with compact sets~$\UU,\VV$ can be quantified via the \emph{duality gap}
\begin{equation}
\label{eq:gap}
\Gap(\bar U, \bar V) := \max_{V \in \VV} f(\bar U,V) -  \min_{U \in \UU} f(U, \bar V).
\end{equation}
Under certain conditions which in particular hold for~\eqref{SaddleUhat}, see~\citet{sion1958}, this CCSPP possesses an optimal solution $W^* = (U^*, V^*)$, called a \textit{saddle point}, for which it holds
$f(U^*,V^*) = \max_{V \in \VV} f(U^*,V) = \min_{U \in \UU} f(U,V^*)$. That is,~$U^*$ (resp.~$V^*$) is optimal in the primal problem of minimizing~$f_{\prim}(U) := \max_{V \in \VV} f(U,V)$ (resp.~the dual problem of maximizing~$f_{\dual}(V) := \min_{U \in \UU} f(U, \bar V)$). 
Thus, the duality gap bounds from above the primal accuracy -- in our case, the accuracy of solving the initial problem~\eqref{eq:erm}.

\subsection{Choice of the Geometry}
\label{sec:norms-and-potentials}
When applied to~CCSPPs, the geometry of mirror descent is specified by the choice of the pair of the primal and dual norms~$\|\cdot\|_{\Unorm}$,~$\|\cdot\|_{\Vnorm}$ in which the optimal solution is likely to be small,
and the \emph{potentials}, or \emph{distance-generating functions} in the terminology of~\citet{nemirovski2004prox}, that must satisfy some compatibility properties with respect to these norms.

\paragraph{Norms.} 
It is natural to use~$\|\cdot\|_1$ as the primal norm~$\|\cdot\|_{\Unorm}$.
On the other hand, due to~\eqref{eq:cube-of-simplices}, it is reasonable to choose the norm~$\|\cdot\|_{\Vnorm}$ within the family of mixed~$\ell_p \times \ell_q$ norms
\begin{equation}
\label{def:mixed-norms}
\|V\|_{p \times q} := \left(\sum_{i} \|V(i,:)\|_q^p \right)^{1/p}, \quad p,q \in [1, \infty],
\end{equation}
i.e., the~$\ell_p$-norm of the vector of~$\ell_q$-norms of the rows, with~$q = 1$. 
The question is how to choose~$p$. The naive choice~$p = \infty$, corresponding to the direct-product structure of~$\VV$, does not allow for a compatible potential (see~\citet{optbook1}).
The remedy is to replace~$p = \infty$ with~$p = 2$, leading to the choice~$\|\cdot\|_\Vnorm = \|\cdot\|_{2\times 1}$.
In fact, this choice can also be motivated from the black-box model perspective~\cite{nemirovsky1983problem}.

\begin{rem}
\label{rem:other-geometries}
Other choices of the regularization norm and the norm~$\|\cdot\|_{\Vnorm}$ are explored in Appendix~\ref{sec:AllGeometries}.
As it turns out, the choice described here is the only one in a broad class of those using the mixed~$\ell_p \times \ell_q$ norms, for which one can achieve the goal stated in Section~\ref{sec:intro}, that is, have both a favorable accuracy guarantee and an efficient (sublinear) algorithmic implementation.
\end{rem}

\paragraph{Partial Potentials.}
\label{sec:partial-potentials}
A potential~$\phi_{\UU}: \UU \to \R \cup \{\infty\}$ is called \textit{compatible} with the norm~$\|\cdot\|_{\Unorm}$ in the sense of~\citet{nemirovski2004prox} when it admits a continous selection of subgradients in the relative interior of~$\UU$, and is~$1$-strongly convex on~$\UU$ with respect to~$\|\cdot\|_{\Unorm}$.
In order to obtain favorable convergence guarantees, the primal and dual potentials~$\phi_{\UU}$,~$\phi_{\VV}$ must satisfy two conditions. 
First, they must be compatible with the chosen norms; second, \textit{the potential differences}, defined by  
\begin{equation}
\label{def:omega-radius}
\begin{aligned}
\Omega_\UU := \textstyle \max_{U\in \UU}\phi_{\UU}(U)  - \min_{U\in \UU}\phi_{\UU}(U), \\ 
\Omega_\VV :=\textstyle  \max_{V\in \VV}\phi_{\VV}(V)  - \min_{V\in \VV}\phi_{\VV}(V),
\end{aligned}
\end{equation}
must be upper-bounded, up to logarithmic factors in the problem dimension, with the squared radii of~$\UU,\VV$ in the corresponding norms~\cite{nemirovsky1983problem}.
We now specify the potentials that satisfy these requirements.

The natural choice for the dual potential~$\phi_{\VV}(\cdot)$, reflecting the product structure of~$\VV$, is the sum of negative entropies~\cite{beck2003mirror}:
\begin{equation}
\label{eq:v-potential-choice}
\phi_{\VV}(V) =  \sum_{i=1}^n\sum_{l=1}^k V_{il}\log(V_{il}). 
\end{equation}
Its compatibility with~$\|\cdot\|_{2 \times 1}$ follows from Pinsker's inequality~\cite{kemperman1969optimum} applied rowwise to~$V$. 
On the other hand, we have
\begin{equation}
\label{eq:v-omega-radius}
\Omega_{\VV} = n \log k,
\end{equation}
whereas the squared~$\|\cdot\|_{2 \times 1}$-norm of any feasible solution~$V$ to~\eqref{MainEquationConstrained} is precisely~$n$. 
Thus,~\eqref{eq:v-potential-choice} is a valid potential on~$\VV$. 

Regarding the choice of the potential on~$\UU$, consider first the unit ``solid'' simplex, i.e., the set~\eqref{eq:solid-simplex} with~$R_{*} = 1$.
On this set, one can define the \textit{unnormalized negative entropy}
\begin{equation}
\label{eq:vN-entropy}
\cH(U) = \sum_{i=1}^d \sum_{j=1}^k  \big\{ U_{ij} \log U_{ij} - U_{ij} \big\}.
\end{equation} 
Clearly,~$\cH(\cdot)$  is continuously differentiable in the interior of its domain, and one can show that it is~$1$-strongly convex on it (see~\citet{yu2013strong}).
Due to Assumption~\ref{ass:radius}, we can consider
\begin{equation}
\label{eq:uhat-potential-choice}
\begin{aligned}
\phi_{\UU}({U})
&:= R_{*}^2 \cdot \cH({U}/{R_{*}}),
\end{aligned}
\end{equation}
which is then a compatible potential on~$\UU$ that satisfies
\begin{equation}
\label{eq:uhat-omega-radius}
\Omega_{\UU} = R_{*}^2 \log(2dk),
\end{equation}
and thus is a valid potential on~$\UU$. Note that for our choice of the potentials, the corresponding Bregman divergences are expressed in terms of the Kullback-Leibler divergence, and thus lead to multiplicative updates. 
This circumstance is crucial for the sublinear algorithm considered in Section~\ref{sec:sublinear}.

\subsection{Composite Saddle-Point Mirror Descent}
\label{sec:non-stoch}
We use the composite variant of saddle-point mirror descent applicable for quasi-bilinear CCSPPs, see~\citet{optbook2,algorec2018arxiv}.
Introducing~$W = (U,V) \in \WW \; [:= \UU \times \VV]$, the algorithm can be summarized as follows.
First, one constructs the {\em joint potential}~$\phi(W)$ on~$\WW$ by reweighting~$\phi_{\UU}(U)$~and~$\phi_{\VV}(V)$: 
\begin{equation}
\label{eq:joint-potential}
\phi(W)  = \frac{\phi_\UU(U)}{2\Omega_\UU}+ \frac{\phi_\VV(V)}{2\Omega_\VV}.
\end{equation}
Such reweighting, possible in our case due to Assumption~\ref{ass:radius}, allows to improve the accuracy bound, replacing the factor~$\Omega_{\UU} + \Omega_{\VV}$ with~$\sqrt{\Omega_{\UU} \Omega_{\VV}}$ (cf.~Theorem~\ref{th:md-deterministic} below).
Then, initializing with~$W^0 = \min_{W \in \WW} \phi_{\WW}(W),$ one iterates
\begin{equation}
\label{MirrorDescent}
\begin{aligned}
W^{t+1}\! = \!\argmin_{W\in\WW} h(W) + \langle G(W^{t}), W \rangle + \frac{D_{\phi}(W,W^{t})}{\gamma_{t}}, 
\end{aligned}
\end{equation}
where~$\{\gamma_t\}$ is the sequence of stepsizes,~$h(W) = \FF(V,Y) + \lambda\, \tr[\ind_{2d \times k}^\top U]$ is the combined ``simple'' term (cf.~\eqref{eq:composite-part-v}),
\begin{equation}
\label{eq:gradient-field}
\begin{aligned}
G(W) &= \left[ \frac{1}{n} X^\top(V-Y),-\frac{1}{n} XU \right]
\end{aligned}
\end{equation}
is the vector field of the partial gradients of~$\wh\Phi(U,V-Y)$, cf.~\eqref{eq:bilinear-part-simplex}, and~$D(\cdot,\cdot)$ is the Bregman divergence\footnotemark~linked to~$\phi(\cdot)$:
\[
D(W,W') = \phi(W) - \phi(W') - \lang \nabla \phi(W'), W-W' \rang.
\]
\footnotetext{We ignore subtleties related to the correct definition of the domain of~$D(W,\cdot)$~(see, e.g.,~\cite{beck2003mirror}); nonetheless, the subsequent algorithms are correctly defined.}
Note that in the case of~\eqref{SaddleUhat}, this amounts to the initialization
\begin{equation}
\label{eq:init-full}
\begin{aligned}
V^0 &= \frac{1}{k} \ind_{n \times k}, \quad U^0 = \frac{R_{*}}{2dk} \ind_{2d \times k},
\end{aligned}
\tag{{\bf Init}}
\end{equation}
and iterations (separable in~$U$ and~$V$) of the form
\begin{equation}
\label{NonStochasticUpdates}
\begin{aligned}
&(U^{t+1},V^{t+1}) \\
&\hspace{-0.2cm}=\argmin\limits_{{U \in \UU, V \in \VV}}
\bigg\{\lambda \, \tr[\ind_{2d \times k}^\top U] + \wh\Phi(U, V^{t}-Y) + \frac{D_{{\UU}}(U, U^{t})}{2\gamma_t\Omega_{\UU}} 
+ \FF(V,Y) - \wh\Phi(U^{t},V) + \frac{D_{{\VV}}(V,V^{t})}{2\gamma_t \Omega_\VV}  \bigg\},
\end{aligned}
\tag{{\bf MD}}
\end{equation}
where~$D_{\UU},D_{\VV}$ are the Bregman divergences for~$\phi_{\UU},\phi_{\VV}$. 

\paragraph{Complexity of Iterations.}
One iteration in~\eqref{NonStochasticUpdates} has running time~$O(dnk)$, and is dominated by the computation of the matrix products~$XU$ and~$X^\top(V-Y)$; once they are known, the proximal step only requires~$O(dk+nk)$ a.o.'s.
Indeed, given~$S^{t} = \frac 1n \wh{X} ^\top(V^{t}-Y)$, the primal update with the choice~\eqref{eq:uhat-potential-choice} of~$\phi_{\UU}$ can be expressed in closed form (refer to~Lemma~\ref{lemmaForU} in Appendix~\ref{sec:lemmas} for the derivation):
\begin{equation}
\label{UpdatesForU}
\begin{aligned}
{U}_{il}^{t+1} = {U}_{il}^{t}  \ {e^{-2\gamma_t S_{il}^{t} R_{*} L}}   \min\left\{e^{-2\gamma_t \lambda R_{*}L}, {R_{*}}/{M}\right\},\\
\text{where} \; \;
L := \log(2dk),  \; \;  \text{and} \; \; M :=  \sum_{i=1}^{2d}\sum_{l=1}^{k}{U}_{il}^{t} \cdot e^{-2\gamma_t S_{il}^{t} R_{*} L}. 
\end{aligned}
\end{equation}
On the other hand, our ability to perform the dual updates depends on the form of~$\ff(v,y)$ in the representation~\eqref{eq:fenchel-loss}. 
In the general case, due to~$\ff(v,y)$ being separable in~$v$, 
we can reduce the dual update to~$O(nk)$ \textit{one-dimensional} optimization problems. 
This can be done by passing to the Lagrangian dual problem (which is separable), minimizing the Lagrangian for the given value of multiplier by solving~$O(nk)$ one-dimensional problems, and finding the optimal Lagrange multiplier via root search.
Moreover, for the multiclass hinge loss~\eqref{eq:loss-hinge} we have the closed-form updates:
\begin{equation}
\label{UpdatesForVSVM}
V_{il}^{t+1} = \frac{V_{il}^{t} \exp\big(2\gamma_t  [\wh{X}{U}^{t} - Y]_{il} \log(k) \big)}{\sum_{l=1}^k V_{il}^{t} \exp\big(2\gamma_t [\wh{X}{U}^{t} - Y]_{il} \log(k) \big)}. 
\end{equation}

\paragraph{Convergence Rate.}
The convergence rate of mirror descent for CCSPPs of the form~\eqref{SaddleUhat} depends on the quantity
\begin{equation}
\label{eq:cross-lip-bilinear}
\cL_{\Unorm,\Vnorm} := \frac{1}{n} \sup_{\| U \|_{\Unorm} \leqs 1} \|\wh XU\|_{\Vnorm^*},
\end{equation}
where~$\|\cdot\|_{\Vnorm^*}$ is the dual norm to~$\|\cdot\|_{\Vnorm}$.
Thus,~$\cL_{\Unorm,\Vnorm}$ is the~$(\Unorm,\Vnorm^*)$-subordinate norm of the linear mapping~$U \mapsto \frac{1}{n}\wh XU$.
For the norms chosen in Section~\ref{sec:norms-and-potentials}, $\cL_{\Unorm,\Vnorm}$ is expressed as a mixed norm~\eqref{def:mixed-norms} (see Appendix for the proof):
\begin{proposition} 
\label{th:lip}
For~$\|\cdot\|_{\Unorm} = \|\cdot\|_{1}$ and~$\|\cdot\|_{\Vnorm} = \|\cdot\|_{2 \times 1}$, one has
\begin{equation}
\label{eq:cross-lip-value}
\cL_{\Unorm,\Vnorm} = \frac{1}{n}\|X^\top\|_{\infty \times 2}.
\end{equation}
\end{proposition}

We obtain the following convergence guarantee for our variant of mirror descent applied to CCSPP~\eqref{SaddleUhat}. 
For simplicity, we consider constant stepsize and simple averaging; empirically we observe similar results for the time-varying stepsize~$\gamma_t \propto {1}/{\sqrt{t+1}}$. 

\begin{theorem}
\label{th:md-deterministic}
 Let~$(\bar{U}^T, \bar{V}^T) = \frac{1}{T}\sum_{t=0}^{T-1} ({U}^t,V^t)$ be the average of the first~$T$ iterates of mirror descent~\eqref{NonStochasticUpdates} with initialization~\eqref{eq:init-full} and stepsize
$
\gamma_t \equiv {1}/({\cL_{\Unorm,\Vnorm}\sqrt{5T\Omega_{\UU}\Omega_\VV}}),
$
with~$\cL_{\Unorm,\Vnorm},~\Omega_{\UU},~\Omega_\VV$ given by~\eqref{eq:cross-lip-value},~\eqref{eq:uhat-omega-radius},~\eqref{eq:v-omega-radius}.
Then the duality gap can be bounded as
\begin{equation}
\label{eq:md-deterministic-result}
\begin{aligned}
\Gap(\bar U^T, \bar V^T) 
&\leqs 
\frac{2\sqrt{5}\cL_{\Unorm,\Vnorm}\sqrt{\Omega_{\UU}\Omega_\VV}}{\sqrt{T}} + \frac{\Rem}{T}\\
&\leqs 
\frac{2\sqrt{5}\|X^\top\|_{\infty\times 2}}{\sqrt{n}}   \frac{\log(2dk) R_{*}}{\sqrt{T}} + \frac{\Rem}{T}.
\end{aligned}
\end{equation}
Here~$\Rem = 0$ for the multiclass hinge loss~\eqref{eq:loss-hinge} and~$\Rem = \max_{y \in \Delta_k}\left\{\ff(\frac{\ind_k}{k},y) \hspace{-0.05cm} - \hspace{-0.05cm}\min_{v \in \Delta_k}\ff(v,y)\right\}$ in the general case; in particular,~$\Rem = O(\log(d))$ for the softmax loss~\eqref{eq:loss-logistic}.
\end{theorem}
\begin{proof}
The first bound follows from the general result for quasi-bilinear CCSPPs, see Theorem~\ref{th:md-general} in Appendix.\footnotemark~
Its combination with~\eqref{eq:v-omega-radius},~\eqref{eq:uhat-omega-radius}, and~\eqref{eq:cross-lip-value} results in~\eqref{eq:md-deterministic-result}.
\footnotetext{Note that the results of~\citet{duchi2010composite} nor those of~\citet{nesterov2013first} cannot be readily applied in our setup.}
\end{proof}

\begin{rem}
\label{rem:l2-moment}
Note that the~$j$-th column~$X_j \in \R^n$ of~$X$ represents the empirical distribution of the feature~$\varphi_j$. 
Hence, when the data is~i.i.d., we have
\begin{equation}
\label{eq:l2-moment-limit}
\frac{\|X^\top\|_{\infty \times 2}}{\sqrt{n}} =  \displaystyle \max_{j \in [d]}  \sqrt{ \frac{\| X_j \|_2^2}{n}} \underset{n \to \infty}{\stackrel{\text{a.s.}}{\longrightarrow}} \max_{j \in [d]} \left(\Edata [\varphi_j^2]\right)^{1/2},
\end{equation}
where the expectation~$\Edata$ is over the data distribution. 
In other words,~${\|X^\top\|_{\infty \times 2}}/{\sqrt{n}}$ has a finite limit, converging to the largest~$L_2$-norm of a feature.
In the non-asymptotic regime,~${\|X^\top\|_{\infty \times 2}}/{\sqrt{n}}$ is the largest \emph{empirical}~$L_2$-norm of a feature, and can be controlled if the features are bounded, or sufficiently light-tailed, via standard concentration inequalities. In particular,
$
{\|X^\top\|_{\infty \times 2}}/{\sqrt{n}} \leqs B
$ 
if the features are uniformly bounded with~$B$, and for the Gaussian features~$\varphi_j \sim \mathcal{N}(0,\sigma_j^2)$ with~$\sigma_j \le B$ we have, with probability at least~$1 - \delta$,
\begin{equation}
\label{eq:l2-moment-gaussian}
\frac{\|X^\top\|_{\infty \times 2}}{\sqrt{n}} \leqs C\sigma \left(1 + \sqrt{\frac{\log(d/\delta)}{n}} \right),
\end{equation}
for some constant~$C$, see~\citet[Lem.~1]{lama2000}.
\end{rem}

\begin{rem}
\label{rem:mp-acceleration}
It is known that the convergence rate can be improved to~$O(1/T)$ for the composite version of Mirror Prox, but this improvement is not preserved in the stochastic setting. 
On the other hand, this allows to emulate the ``mini-batching'' technique, by sampling the matrix products repeatedly (or in parallel), 
and controlling the variability of the averaged gradient estimates via Nemirovski's inequalities in the vein of~\citet[Sec.~2.5.1]{optbook2}.
\end{rem}

\section{Sampling Schemes}
\label{sec:sampling}
Recall that the bottleneck of mirror descent iterations~\eqref{NonStochasticUpdates} is computing the matrix products~$\wh X U^{t}, \wh X^\top(V^{t}-Y)$ which requires~$O(dnk)$ a.o.'s.
Inspired by~\citet{optbook2}, we propose sampling schemes that produce unbiased estimates $\xi_{U^t}$ and~$\eta_{V^t,Y}$ of $\wh X U^t$ and~$\wh X^\top(V^t-Y)$ with reduced complexity of computation, 
and use them to approximate the true partial gradients, arriving at the following variant of stochastic mirror descent:
\begin{equation}
\label{StochasticUpdates}
\begin{aligned}
&(U^{t+1},V^{t+1}) \\
&\hspace{-0.2cm}=\argmin\limits_{{U \in \UU, V \in \VV}}
\Big\{\lambda \, \tr[\ind_{2d \times k}^\top U] + \frac{1}{n} \tr\left[\eta_{V^{t},Y} {U} \right]+
\frac{D_{{\UU}}(U, U^{t})}{2\gamma_t\Omega_{\UU}} \\
& \hphantom{\lambda \tr[\ind_{2d \times k}^\top U]} 
+ \FF(V,Y) - \frac{1}{n}\tr \left[\xi_{U^{t}}^\top V \right] + \frac{D_{{\VV}}(V,V^{t})}{2\gamma_t \Omega_\VV}  \Big\}.
\end{aligned}
\tag{{\bf SMD}}
\end{equation}
Since the gradients are now replaced with their unbiased estimates, the accuracy bound gets augmented with an extra term that reflects the variability of these estimates.
This extra term is known to be~$O(\sqrt{(\Omega_{\UU} \sigma_{\VV}^2 + \Omega_{\VV} \sigma_{\UU}^2)/T}),$
where~$\sigma_{\UU}^2$ and~$\sigma_\VV^2$ are ``variance proxies'' -- the natural analogues of the variances of~$\xi_{U}$ and~$\eta_{V,Y}$ for the chosen norms:
\begin{equation}
\label{def:variance-proxies}
\begin{aligned}
\sigma^2_{\UU} &:= \frac{1}{n^2}\sup_{U \in \UU} \E\Big[\big\|\wh{X}U - \xi_{U} \big\|_{\Vnorm^*}^2\Big], \\
\sigma^2_{\VV} &:= \frac{1}{n^2}\!\!\sup_{(V,Y) \in \VV \times \VV}\!\! \E\Big[\big\|\wh{X}^{\top}(V-Y) - \eta_{V,Y} \big\|_{\Unorm^*}^2\Big].
\end{aligned}
\end{equation}
We consider two sampling schemes for~$\xi_{U}$ and~$\eta_{\vphantom{\wh U}V,Y}$: \textit{partial sampling} where the estimates are obtained by sampling the rows of~$U$ and~$V-Y$, and \textit{full sampling}, where one subsequently samples their columns.
In both cases, we derive the data-dependent sampling disrtibutions with near-optimal variance proxies.
Our finding is that for the mirror descent geometry chosen in Sec.~\ref{sec:geometry-choice-main}, and under mild assumptions on the data distribution, application of both schemes with the found distributions 
results in essentially the same convergence rate as in the deterministic case.

\subsection{Partial Sampling Scheme}
\label{sec:part-stoch}
In the \textit{partial sampling scheme}, we choose a pair of distributions~$p = (p_1,..., p_{2d}) \in \Delta_{2d}$ and~$q = (q_1,..., q_{n}) \in \Delta_{n}$, and draw one row of~$U$ and~$V-Y$ (i.e., a feature and a tratining example) at a time according to~$p$ and~$q$.
In other words, we produce the estimates 
\begin{equation}
\label{eq:partial-sampling}
\xi_{U}(p) = \wh X \dfrac{e_i^{\vphantom\top}e_i^\top}{p_i} U,  \quad 
\eta_{V,Y}(q) = \wh X^\top\dfrac{e_j^{\vphantom\top}e_j^\top}{q_j}(V-Y),
\tag{\bf Part-SS}
\end{equation}
where~$e_i \in \Delta_{2d}$ and $e_j \in \Delta_{n}$ are standard basis vectors, and~$i \in [2d], j \in [n]$ are drawn from~$p, q$ correspondingly; clearly, this gives unbiased estimates.
The challenge is to choose the distributions~$p,q$. 
In the Euclidean case, i.e., when~$\|\cdot\|_{\Unorm^*},\|\cdot\|_{\Vnorm^*}$ are Frobenius norms, one can explicitly minimize the variances of the resulting estimates, and it is equivalent to minimizing the second moments~$\E \left[\Vert \xi_{U}(p)\Vert_{\F}^2 \right], \E \left[\Vert \eta_{V,Y}(q)\Vert_{\F}^2\right]$. 
In general, this is not the case. 
Next we show that for our mixed norms, the problem of minimizing the second moment proxies, i.e., finding
\begin{equation}
\label{eq:optimal-prob}
\begin{aligned}
p^{\ast} = p^{\ast}(\wh X, U) 
&\in  \Argmin_{p\in\Delta_{2d}} \! \E \left[\Vert \xi_{U}(p)\Vert_{\Vnorm^\ast}^2 \right], \\
\!\!\!\!q^{\ast} = q^\ast(\wh X, V,Y) 
&\in \Argmin_{q\in\Delta_n} \! \E \left[\Vert \eta_{V,Y}(q)\Vert_{\Unorm^\ast}^2\right],\!\!\!
\end{aligned}
\end{equation}
can be solved explicitly, due to the matrices in the right-hand side of~\eqref{eq:partial-sampling} being one-rank.
The variance proxies (cf.~\eqref{def:variance-proxies}) can then be bounded via the triangle inequality.

\begin{proposition}
\label{th:variance-partial}
When~$\|\cdot\|_{\Unorm} = \|\cdot\|_{1}$ and~$\|\cdot\|_{\Vnorm} = \|\cdot\|_{2 \times 1}$, 
the optimal distributions~$p^{\ast}$  and~$q^\ast$, cf.~\eqref{eq:optimal-prob}, are given by
\begin{equation}
\label{eq:optimal-probabilities-partial}
\begin{aligned}
p^{\ast}_i &\propto {\|\wh{X}(:,i) \|_{2} \cdot  \| {U}(i,:) \|_{\infty}},\\
q^{\ast}_j  &\propto {\| \wh{X}(j,:) \|_{\infty}\cdot  \| V(j,:) - Y(j,:)\|_{\infty}},
\end{aligned}
\end{equation}
where~$A(i,:)$ and~$A(:,j)$ are the~$i$-th row and~$j$-th column of~$A$.
Moreover, the corresponding variance proxies satisfy
\begin{equation}
\label{eq:variance-bounds-partial}
\begin{aligned}
\sigma^2_{\UU}(p^*) &\leqs \frac{4 R_{*}^2 \|X^\top\|_{\infty\times 2}^2}{{n^2}}, \\
\sigma_{\VV}^2(q^*) &\leqs  \frac{8\|X^\top\|_{\infty\times 2}^2}{n}  +\frac{8\|X\|_{1\times\infty}^2}{n^2}.
\end{aligned}
\end{equation}
\end{proposition}
See Appendix~\ref{sec:proof-variance-partial} for the proof of an extended result for the general mixed norms~\eqref{def:mixed-norms}.
Combined with the general result for composite saddle-point stochastic mirror descent (Theorem~\ref{th:md-general-stochastic} in Appendix), Proposition~\ref{th:variance-partial} implies the following result:
\begin{theorem}
\label{th:md-partial}
Let~$(\bar{U}^T, \bar{V}^T) = \frac{1}{T}\sum_{t=0}^{T-1} ({U}^t,V^t)$ be the average of~$T$ iterates of stochastic mirror descent~\eqref{StochasticUpdates} initialized with~\eqref{eq:init-full}, equipped with sampling scheme~\eqref{eq:partial-sampling} with distributions~\eqref{eq:optimal-probabilities-partial}, and with stepsize
\begin{equation}
\label{eq:md-stochastic-step}
 \gamma_t \equiv \frac{1}{\sqrt{2T}} \min \Big\{ \frac{1}{\cL_{\Unorm,\Vnorm}\sqrt{5\Omega_{\UU}\Omega_{\VV}}}, \frac{1}{\sqrt{\Omega_{\UU} \bar\sigma_{\VV}^2 + \Omega_{\VV} \bar\sigma_{\UU}^2}}\Big\},
\end{equation}
with~$\cL_{\Unorm,\Vnorm},\Omega_{\UU},\Omega_\VV$ given by~\eqref{eq:cross-lip-value},~\eqref{eq:v-omega-radius},~\eqref{eq:uhat-omega-radius}, and the upper bounds~$\bar\sigma_{\UU}^2, \bar\sigma_{\VV}^2$ on the variance proxies given by~\eqref{eq:variance-bounds-partial}.
Then
\begin{equation}
\label{eq:md-partial-result}
\begin{aligned}
\E[\Gap(\bar U^T, \bar V^T)]
&\leqs  
\frac{2\sqrt{10}\cL_{\Unorm,\Vnorm}\sqrt{\Omega_{\UU}\Omega_{\VV}}}{\sqrt{T}} + \frac{2\sqrt{2}\sqrt{\Omega_{\UU} \bar\sigma_{\VV}^2 + \Omega_{\VV} \bar\sigma_{\UU}^2}}{\sqrt{T}} + \frac{\Rem}{T} \\
&\leqs  
\Big( \frac{16.2\, \|X^\top\|_{\infty\times 2}}{\sqrt{n}} + \frac{8\|X \|_{1\times\infty}}{n}\Big) \frac{\log(2dk) R_{*}}{\sqrt{T}} + \frac{\Rem}{T},
\end{aligned}
\end{equation}
where the expectation is over the randomness of the algorithm, and~$\Rem$ is the same as in Theorem~\ref{th:md-deterministic}.
\label{ConvergenceRatesPartlyStochastic}
\end{theorem}

\begin{rem}
\label{rem:light-tailed}
Comparing~\eqref{eq:md-partial-result} with~\eqref{eq:md-deterministic-result}, we see that the partial sampling~\eqref{eq:partial-sampling} does not deteriorate the convergence rate as long as the extra term~$\|X \|_{1\times\infty}/n$ does not dominate~$\|X^\top\|_{\infty\times 2}/\sqrt{n}$. 
When the data is~\emph{light-tailed}, the two terms are comparable. 
Indeed,~$\|X^\top\|_{\infty\times 2}/\sqrt{n}$ a.s.~converges to~$\max_{j \in [d]} \left(\Edata [\varphi_j^2]\right)^{1/2}$, where~$\varphi_j$'s are the features~(cf.~\eqref{eq:l2-moment-limit}), and the term~$\|X \|_{1 \times \infty}/{n}$ clearly converges to~$\Edata \max_{j \in [d]} |\varphi_j|$.
When~$\varphi_j$'s are subgaussian, we have 
\[
\begin{aligned} 
\Edata [\max_{j \in [d]} |\varphi_j|] 
& \leqs \wt O_d(1) \max_{j \in [d]} (\Edata [\varphi_j^2])^{{1}/{2}},
\end{aligned}
\]
where~$\wt O_d(1)$ is a log-factor in~$d$. 
Similar conclusions hold in finite sample: both terms admit the same bound in terms of the uniform bound on the features (cf. Remark~\ref{rem:l2-moment}), and when~$\varphi_j \sim \mathcal{N}(0,\sigma_j^2)$ 
with~$\sigma_j \leqs \sigma$ for any~$j \in [d]$, we have  
\[
{\|X \|_{1\times\infty}}/{n} \leqs C\sigma \sqrt{\log(dn/\delta)},
\] 
w.p.~$1-\delta$, with a similar bound for~$\|X^\top\|_{\infty \times 2}/\sqrt{n}$, cf.~\eqref{eq:l2-moment-gaussian}.
\end{rem}

\paragraph{Complexity.}
$O(dn)$ a.o.'s are needed once to compute the row and column norms of~$\wh X$ in~\eqref{eq:optimal-probabilities-partial}.
Producing~$\xi_{U}(p^*),\eta_{V,Y}(q^*)$ costs~$O(dk+nk)$ a.o.'s, including the computation of the distibutions~\eqref{eq:optimal-probabilities-partial}; given them, the proximal step has the same complexity 
as discussed in Sec.~\ref{sec:geometry-choice-main}.

\subsection{Full Sampling Scheme}
\label{sec:full-stoch}

In the \emph{full sampling scheme}, sampling of the rows of~$U$ and~$V-Y$ is augmented with a subsequent column sampling:
\begin{equation}
\label{eq:full-sampling}
\begin{aligned}
\xi_{U}(p,P) &= \wh X \frac{e_i^{\vphantom\top}e_i^\top}{p_i} U \frac{e^{\vphantom \top}_l e_l^{\top}}{P_{il}},   \\[-.1cm]
\eta_{V,Y}(q,Q) &= \wh X^\top\dfrac{e_j^{\vphantom\top}e_j^\top}{q_j}(V-Y) \frac{e^{\vphantom \top}_l e_l^{\top}}{Q_{jl}},
\end{aligned}
\tag{\bf Full-SS}
\end{equation}
where~$i \in [2d]$ and~$j \in [n]$ are drawn from distributions~$p \in \Delta_{2d}, q \in \Delta_{n}$ as before, and the rows of the matrices~$P \in \Delta_k^{\otimes 2d}$ and~$Q \in \Delta_k^{\otimes n}$ specify the conditional sampling distribution of the class~$l \in [k]$ given~$i$ and~$j$. Unbiasedness of these estimates is easy to verify.
Next we derive the optimal sampling distributions and bound their variance proxies (refer to Appendix~\ref{Theorem:variancesFull} for the proof).

\begin{proposition}
\label{th:variance-full}
Let~$\|\cdot\|_{\Unorm} = \|\cdot\|_{1 \times 1}$,~$\|\cdot\|_{\Vnorm} = \|\cdot\|_{2 \times 1}$.
The optimal solutions~$(p*,P^*),~(q^*,Q^*)$ to
\[
\! \min_{ p \in \Delta_{2d}, P \in \Delta_k^{\otimes 2d}}  \E \| \xi_{U}(p,P) \|_{\Vnorm^*}^2, 
\quad
\min_{q \in\Delta_n,  Q \in \Delta_k^{\otimes n}} \E \| \eta_{V,Y}(q,Q) \|_{\Unorm^*}^2
\] 
are unique and given by
\begin{equation}
\label{eq:optimal-probabilities-full}
\begin{aligned}
p^{\ast}_i &\propto \| \wh{X}(:,i) \|_{2} \,  \|{U}(i,:) \|_{1}, \hspace{-0.1cm}
&P_{il}^* &\propto {U}_{il}; \\
q^{\ast}_j &\propto \| \wh{X}(j,:) \|_{\infty} \,  \| V(j,:) - Y(j,:)\|_{1}, \hspace{-0.1cm}
&Q_{jl}^* &\propto {|V_{jl} - Y_{jl}|}.
\end{aligned}
\end{equation}
The respective variance proxies still admit the bounds~\eqref{eq:variance-bounds-partial}.
\end{proposition}

\begin{corollary}
\label{th:md-full}
The second bound of~\eqref{eq:md-partial-result} in~Theorem~\ref{th:md-partial} remains true when we replace~\eqref{eq:partial-sampling} with the sampling scheme~\eqref{eq:full-sampling} with distributions given by~\eqref{eq:optimal-probabilities-full}.
\end{corollary}
\section{Sublinear Algorithm for Multiclass $\ell_1$-SVM}
\label{sec:sublinear}

For the hinge loss~\eqref{eq:loss-hinge}, we provide a \textit{sublinear} implementation of the bundle~$\eqref{StochasticUpdates}+\eqref{eq:full-sampling}$ with sampling distributions~\eqref{eq:optimal-probabilities-full}.

\paragraph{Lazy Updates.}
Note that although the estimates~$\xi_{U},\eta_{V,Y}$ produced in~\eqref{eq:full-sampling} are sparse (each contains a single non-zero column), the updates in~\eqref{StochasticUpdates}, which can be expressed as~\eqref{UpdatesForU}--\eqref{UpdatesForVSVM} with~$\xi_{U},\eta_{V,Y}$ instead of the corresponding matrix products, are \textit{dense}, and implementing them naively costs~$O(dk+nk)$ a.o.'s.
Fortunately, these updates have a special form: all elements in each row of~$U^t$ and~$V^t$ are simply rescaled with the same factor -- except for at most two elements corresponding to a single non-zero element of~$\eta_{V^t,Y}$ and at most two non-zero elements of~$\xi_{U^t} - Y$ in this row. 
To exploit this fact, we perform ``lazy'' updates:
instead of explicitly computing the actual iterates~$(U^t,V^t)$, we maintain the quadruple~$(\wt U, \alpha,\wt V, \beta)$, where~$\wt U,\wt V$ have the same dimensions as~$U,V$, while~$\alpha \in \R^{2d}$ and~$\beta \in \R^{n}$ are the ``scaling vectors'', so that at any iteration~$t$ it holds
\begin{equation}
\label{eq:scaling-invariant}
U^t(i,:) = \wt U(i,:) \cdot \alpha(i), \quad V^t(j,:) = \wt V(j,:) \cdot \beta(j)
\end{equation}
for any row of~$U^t$ and~$V^t$.
Initializing with~$(\wt U,\wt V) = (U,V)$,~$\alpha = \ind_{2d}$,~$\beta = \ind_{n}$, we can update the whole quadruple, while maintaining~\eqref{eq:scaling-invariant}, by updating at most two elements in each row of~$\wt U$ and~$\wt V$, and encapsulating the overall scaling of rows in~$\alpha$ and~$\beta$. Clearly, this update requires only~$O(d+n)$ operations once~$\xi_{U^t},\eta_{V^t,Y}$ have been drawn.

\paragraph{Sampling.}
Computing the distributions~$p^*,q^*$ from~\eqref{eq:optimal-probabilities-full} requires the knowledge of~$\|\wh{X}(:,i)\|_2$ and~$\|\wh{X}(j,:)\|_{\infty}$ which can be precomputed in~$O(dn)$ a.o.'s, and maintaining $O(d+n)$ norms~$\pi_i, \rho_j$ of the rows of~$U^t$ and~$V^t - Y$ that can maintained in~$O(1)$ a.o.'s each using~\eqref{eq:scaling-invariant}. 
Thus,~$p^*$ and~$q^*$ can be updated in~$O(d+n)$. Once it is done, we can sample~$i^t \sim p^*$ and~$j^t \sim q^*$, and then sample the class from~$P^*$ and~$Q^*$, cf.~\eqref{eq:optimal-probabilities-full}, by computing only the~$i^t$-th row of~$P^*$ and the~$j^t$-th row of~$Q^*$, both in~$O(k)$ a.o.'s. Thus, the total complexity of producing~$\xi_{U^t},\eta_{V^t,Y}$ is~$O(d+n+k)$.

\paragraph{Tracking the Averages.}
Similar ``lazy'' updates can be performed for the running averages of the iterates. Omitting the details, this requires~$O(d+n)$ a.o.'s per iteration, plus post-processing of~$O(dk+nk)$ a.o.'s.\\

The above ideas are implemented in Algorithm~\ref{alg:main} whose correctness is formally shown in Appendix~\ref{sec:alg-correctness} (see also Sec.~\ref{sec:alg-details} for an additional discussion).
Its close inspection shows the iteration cost of~$O(d+n+k)$ a.o.'s, plus~$O(dn+dk+nk)$ a.o.'s for pre/post-processing, and the memory complexiy of~$O(dn+dk+nk)$. Moreover, the term~$O(dk)$, which dominates in high-dimensional and highly multiclass problems, can be removed if one exploits sparsity of the corresponding primal solution to the~$\ell_1$-constrained problem~\eqref{MainEquationConstrained}, and outputs it directly, bypassing the explicit storage of~$\wt U$ (see~Appendix~\ref{sec:alg-correctness} for details). 
Note that when~$n = O(\min(d,k))$, the resulting algorithm enters the sublinear regime after as few as~$O(n)$ iterations.

\begin{algorithm}[!b]
\caption{Sublinear Multiclass~$\ell_1$-Regularized SVM}
\label{alg:main}
\begin{algorithmic}[1]
\REQUIRE \mbox{$X \in \R^{n \times d}$,~$y \in [k]^{\otimes n}$,~$\lambda$,~$R_*$, $T \geqs 1$,~$\{\gamma_t\}_{t = 0}^{T-1}$}
\STATE Obtain~$Y \in \Delta_k^{\otimes n}$ from the labels~$y$; \quad $\wh X \equiv [X,-X]$
\STATE $\alpha \gets \ind_{2d}$; \;\; $\wt U \gets \frac{R_{*}\ind_{2d \times k}}{2dk}$; \;\; $\beta \gets \ind_{n}$;  \;\; $\wt V \gets \frac{\ind_{n \times k}}{k}$ 
\label{alg:line-U-init}
\FOR{$\imath=1$ \TO $2d$}
\STATE $\sigma({\imath}) \equiv \|\wh{X}(:,\imath)\|_2$; \;\;\, $\pi(\imath) \gets \|\wt{U}(\imath,:)\|_1$
\ENDFOR
\FOR{$\jmath=1$ \TO $n$}
\STATE $\tau(\jmath) \equiv \|\wh{X}(\jmath,:)\|_\infty$; \;  $\rho(\jmath) \gets \|\wt{V}(\jmath,:) - Y(\jmath,:)\|_1$ 
\ENDFOR
\item[\#] {\em Initialize machinery to track the cumulative sums}
\STATE $\Usum  \gets 0_{2d\times k}$; \;\, $\Vsum \gets 0_{n\times k}$ \COMMENT{Cumulative sums}
\STATE $\Acur \gets \zeros_{2d}$; \; $\Bcur \gets 0_{n}$; \; $\Aprev \gets 0_{2d\times k}$; \; $\Bprev \gets 0_{n\times k}$
\FOR[\eqref{StochasticUpdates} iterations]{$t=0$ \TO $T-1$}
\STATE Draw~$\jt \sim \tau \circ \rho$ \COMMENT{$\circ$ is the elementwise product}
\STATE Draw~$\lt \sim |\wt V(\jt,:) \cdot \beta_\jt   - Y(\jt,:)|$ 
\STATE $[\Usum, \Aprev, \Acur] \gets \textsc {TrackPrimal}(\wt U, \Usum, \Aprev, A,\alpha,\lt)$
\item[\#] {\em The only non-zero column of~$\eta_{V^t,Y}$, cf.~\eqref{eq:full-sampling}:}
\STATE ${\eta} \gets \wh{X}(\jt,:) \cdot \frac{\sum_{\jmath=1}^n \tau(\jmath) \cdot \rho(\jmath)\cdot\sgn[\beta_{\jmath} \cdot V(\jmath,\lt)-Y(\jmath,\lt)]}{\tau(\jt)}$
\label{alg:line-eta}
\STATE $[\wt{U}, \alpha, \pi] \gets \textsc{UpdatePrimal}(\wt U, \alpha, \pi, \eta, \lt, \gamma_t, \lambda, R_*)$ 
\STATE Draw $\it  \sim \sigma \circ \pi$ 
\STATE Draw $\ellt \sim \wt U(\it,:)$
\STATE $[\Vsum, \Bprev, \Bcur] \gets \textsc {TrackDual}(\wt V, \Vsum, \Bprev, B, \beta, \ellt, y)$
\item[\#] {\em The only non-zero column of~$\xi_{U^t}$, cf.~\eqref{eq:full-sampling}:}
\STATE $\xi \gets  \wh{X}(:,\it) \cdot \frac{\sum_{\imath=1}^{2d}  \sigma(\imath) \cdot \pi(\imath)}{\sigma(\it)}$
\label{alg:line-xi}
\STATE $[\wt{V}, \beta, \rho] \gets \textsc{UpdateDual}(\wt V, Y, \beta, \rho, \xi, \ellt, y, \gamma_t)$ 
\ENDFOR
\FOR[Postprocessing of cumulative sums]{$l = 1$ \TO $k$}
\STATE $\Usum(:,l) \gets \Usum(:,l) + \wt{U}(:,l) \circ (\alpha + \Acur - \Aprev(:,l))$
\label{alg:line-post-prim}
\STATE $\Vsum(:,l) \gets \Vsum(:,l) + \wt{V}(:,l) \circ (\beta + \Bcur - \Bprev(:,l))$
\label{alg:line-post-dual}
\ENDFOR
\ENSURE $\frac{1}{T+1}\Usum, \frac{1}{T+1} \Vsum$ \COMMENT{Averages~$(\bar U^{T+1}, \bar V^{T+1})$}
\end{algorithmic}
\end{algorithm}

\resetProcCounter

\begin{algorithm}
\caption{\sc UpdatePrimal}
\label{proc:primal-update}
\begin{algorithmic}[1]
\REQUIRE $\wt U \in \R^{2d \times k}$,~$\alpha, \pi,\eta \in \R^{2d}$,~$l \in [k]$,~$\gamma$,~$\lambda$,~$R_*$
\STATE $L \equiv \log(2dk)$
\FOR {$i = 1$ \TO $2d$}
\STATE $\mu_i  = \pi_i - {\alpha_i \cdot \tilde{U}(i,l)} \cdot (1-e^{-2\gamma L R_* \eta_i/n})$
\ENDFOR
\STATE $M = \sum_{i=1}^{2d} \mu_i$
\STATE $\nu= \min\{e^{-2\gamma L {R}_\ast \lambda}, {R}_\ast /M\}$
\FOR {$i = 1$ \TO $2d$}
\STATE $\wt U(i,l) \gets \wt{U}(i,l) \cdot e^{-2\gamma L R_* \eta_i / n}$
\STATE $\alpha^{+}_i = \nu \cdot \alpha_i$
\STATE $\pi_i^{+} = \nu \cdot \mu_i$
\ENDFOR
\ENSURE $\wt U, \alpha^{+},\pi^{+}$ 
\end{algorithmic}
\end{algorithm}

\begin{algorithm}
\caption{\sc UpdateDual}
\label{proc:dual-update}
\begin{algorithmic}[1]
\REQUIRE $\wt V,Y \in \R^{n \times k}$,~$\beta, \rho, \xi \in \R^{n}$,~$\ell \in [k]$,~$y \in [k]^{\otimes n}$,~$\gamma$
\STATE $\theta = e^{-2\gamma \log(k)}$
\FOR{$j = 1$ \TO $n$}
\STATE $\omega_j = e^{2 \gamma \log(k) \xi_j}$ 
\STATE $\veps_j = e^{-2\gamma \log(k)Y(j,\ell)}$
\STATE $\chi_j = 1-\beta_j \cdot \wt{V}(j,\ell) \cdot (1-\omega_{j} \cdot \veps_j)$
\IF[not the actual class of~$j$ drawn]{$\ell \ne y_j$}
\STATE $\chi_j \gets \chi_j - \beta_j \cdot \wt{V}(j,y_j) \cdot (1 - \theta)$
\ENDIF
\label{alg:line-chi}
\STATE $\beta_j^+ = \beta_j / \chi_j$
\label{alg:line-newbeta}
\STATE $\wt{V}(j,\ell) \gets \wt{V}(j,\ell) \cdot \omega_j \cdot \veps_j$
\label{alg:line-newVl}
\STATE $\wt{V}(j,y_j) \gets \wt{V}(j,y_j) \cdot \omega_j \cdot \theta$
\label{alg:line-newVyj}
\STATE $\rho^{+}_{j} = 2 - 2\beta_j^{+}\cdot \wt{V}(j,y_j)$
\label{alg:line-newrho}
\ENDFOR
\ENSURE $\wt V, \beta^+, \rho^+$ 
\end{algorithmic}

\end{algorithm}

\begin{algorithm}
\caption{\sc TrackPrimal}
\label{proc:primal-track}
\begin{algorithmic}[1]
\REQUIRE $\wt U, \Usum, \Aprev  \in \R^{2d \times k}$,~$\Acur, \alpha \in \R^{2d}$,~$l \in [k]$
\FOR{$i = 1$ \TO $2d$}
\STATE \mbox{$\Usum(i,l) \gets \Usum(i,l) + \wt{U}(i,l) \cdot ( \Acur_i + \alpha_i -  \Aprev(i,l))$}
\label{alg:line-Usum-update}
\STATE $\Aprev(i,l) \gets \Acur_i + \alpha_i$
\label{alg:line-Aprev-update}
\STATE $\Acur_i \gets \Acur_i + \alpha_i$
\label{alg:line-Acur-update}
\ENDFOR
\ENSURE $\Usum$,~$\Aprev$,~$\Acur$ 
\end{algorithmic}
\end{algorithm}

\begin{algorithm}
\caption{\sc TrackDual}
\label{proc:dual-track}
\begin{algorithmic}[1]
\REQUIRE \hspace{-0.1cm} \mbox{$\wt V, \Vsum, \Bprev  \in \R^{n \times k}$,~$\Bcur,\beta \in \R^{n}$,~$\ell \in [k],$~$y \in [k]^{\otimes n}$}
\FOR{$j = 1$ \TO $n$}
\FOR[$\{\ell, y_j \}$ has 1 or 2 elements]{$l \in \{\ell, y_j \}$}
\STATE $\Vsum(j,l) \gets \Vsum(j,l) + \wt{V}(j,l) \cdot ( \Bcur_j + \beta_j -  \Bprev(j,l))$
\STATE $\Bprev(j,l) \gets \Bcur_j + \beta_j$
\ENDFOR
\STATE $\Bcur_j \gets \Bcur_j + \beta_j$
\ENDFOR
\ENSURE $\Vsum$,~$\Bprev$,~$\Bcur$ 
\end{algorithmic}
\end{algorithm}

\section{Experiments} 
\label{sec:experiments}

\paragraph{Sublinear Runtime.}
To illustrate the sublinear iteration cost of Algorithm~\ref{alg:main}, we consider the following experiment.
Fixing~$n=d=k$, we generate $X$ with i.i.d.~standard Gaussian entries, take~$U^o$ to be the identity matrix (thus very sparse), and generate the labels by~$\argmax_{l \in [k]} x_j U^o_l +\frac{1}{\sqrt{d}}\mathcal{N}(0,I_d),$ where~$x_j$'s are the rows of~$X$, and~$U^o_l$'s are the columns of~$U^o$. 
This is repeated~$10$ times with~$n = d = k$ increasing by a constant factor~$\kappa$; each time we run Algorithm~\ref{alg:main} for a fixed (large) number of iterations to dominate the cost of pre/post-processing, with~$R_* = \|U^o\|_1$ and~$\lambda = 10^{-3}$, and measure its runtime. We observe (see~Tab.~\ref{TableLinear}) that the runtime is proportional to~$\kappa$, as expected.

\begin{table}
\begin{center}
\begin{tabular}{|l|c|c|c|c|c|}
\hline
$n=d=k$ 
&400&800&1600&3200&6400\\
\hline
$T = 10^4$ 
& 1.17& 2.07& 4.27& 7.55& 15.56\\
\hline
$T = 2 \cdot 10^4$ 
&2.47& 4.27& 8.74& 14.65& 30.77\\
\hline
\end{tabular}
\end{center}

\vspace*{-.5cm}

\caption{Runtime (in seconds) of Algorithm~\ref{alg:main} in a synthetic data experiment.}
\label{TableLinear}

\end{table}

\paragraph{Synthetic Data Experiment.}
We compare Algorithm~\ref{alg:main} with two competitors:~$\|\cdot\|_1$-composite stochastic subgradient method (SSM) for the primal problem~\eqref{eq:erm}, in which one uniformly samples one training example at a time~\cite{shalev2011pegasos}, leading to~$O(dk)$ iteration cost; deterministic saddle-point Mirror Prox (MP) with geometry chosen as in Algorithm~\ref{alg:main}, for which we have~$O(dnk)$ cost of iterations but~$O(1/T)$ convergence in terms of the number of iterations. We generate data as in the previous experiment, fixing~$n = d = k = 10^3$. The randomized algorithms are run~$10$ times for~$T \in \{ 10^{m/2}, m = 1, ..., 12\}$ iterations with constant stepsize (we use stepsize~\eqref{eq:md-stochastic-step} in Algorithm~\ref{alg:main}, choose the one recommended in Theorem~\ref{th:md-deterministic} for MP, and use the theoretical stepsize for SSM, explicitly computing the variance of subgradients and the Lipschitz constant).
Each time we compute the duality gap and the primal accuracy, and measure the runtime (see~Fig.~\ref{fig:acc-simulated}). We see that Algorithm~\ref{alg:main} outperforms SSM, which might be the combined effect of sublinearity and our choice of geometry. 
It also outmatches MP up to high accuracy due to the sublinear effect (MP eventually ``wins'' because of its~$O(1/T)$ rate).\footnotemark 

\footnotetext{The codes of our experiments are available online at  \url{https://github.com/flykiller/sublinear-svm}.}

\newcommand\halfwth{0.48\textwidth}
\newcommand\halfhgt{0.24\textheight}
\newcommand\thirdwth{0.32\textwidth}
\newcommand\thirdhgt{0.16\textheight}
\newcommand\quartwth{0.238\textwidth}
\newcommand\quarthgt{0.119\textheight}
\newcommand\sixthwth{0.156\textwidth}
\newcommand\sixthhgt{0.078\textheight}

\begin{figure}[H]
\begin{minipage}{\halfwth}
\includegraphics[width=1.00\textwidth]{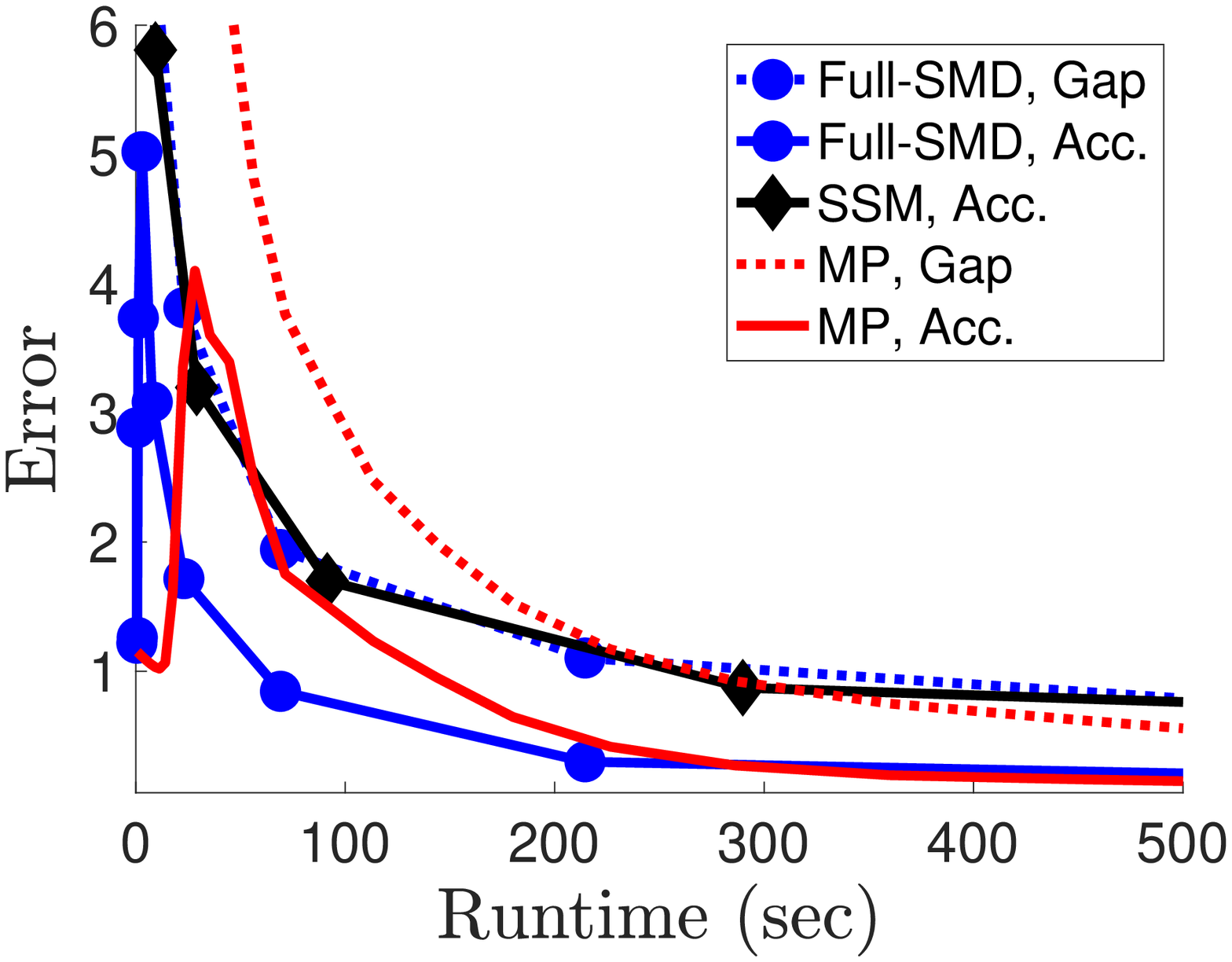}
\end{minipage}
\begin{minipage}{\halfwth}
\includegraphics[width=1.00\textwidth]{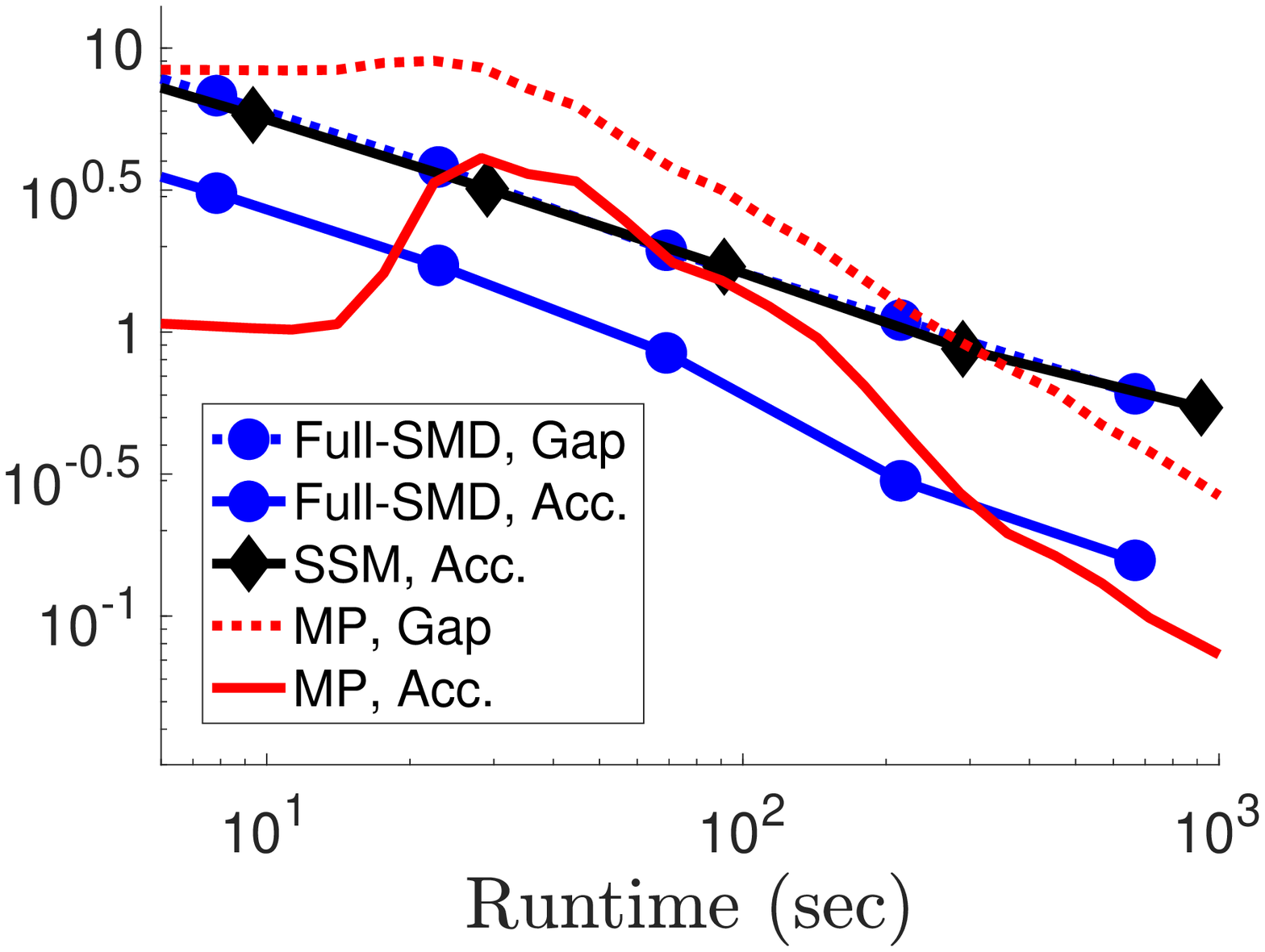}
\end{minipage}
\caption{
Primal accuracy and duality gap (when available) for Algorithm~\ref{alg:main}, stochastic subgradient method (SSM), and Mirror Prox (MP) with exact gradients, on a synthetic data benchmark, in the natural scale (left) and the log-log scale (right).}
\label{fig:acc-simulated}
\end{figure}

\section{Conclusion and Perspectives}
We proposed efficient algorithms based on stochastic mirror descent with entropy-type potentials, that allows to train~$\ell_1$-regularized multiclass linear classifiers in the case when the loss admits an explicit Fenchel-type representation, by reducing the finite-sum minimization problem to its saddle-point equivalent. In particular, in the case of the multiclass hinge loss we were able to construct a sublinear algorithm with the cost~$O(d+n+k)$ of iterations, which was possible due to the multiplicative form of the updates. 
We envision the following directions for future work.
\begin{itemize}
\item It would be interesting to investigate whether out approach can also yield sublinear algorithms for other Fenchel-Young losses, in particular for the multiclass logistic model~\eqref{eq:loss-logistic}, which is widely used in Natural Language Processing (NLP) problems, where $d$, $n$ and $k$ are on order of millions or even billions~\citep{chelba2013one,partalas2015lshtc}.
\item It would be useful to implement more flexible stepsizes, including the online stepsize search in the vein of~\citet{optbook1}, and to conduct larger scale experiments, including those on real data.
\item According to our result in Theorem~\ref{th:md-partial} and Corollary~\ref{th:md-full}, the loss of accuracy due to sampling is negligible when the data is light-tailed, but otherwise, the additional error due to sampling might be significant (see~Remark~\ref{th:md-full}). On the other hand, in the case of (non-composite) bilinear problems with vector variables,~\citet[Sec.~2.5.2.3]{optbook2} propose a technique of transforming the problem to an equivalent one, for which the loss of accuracy due to sampling is \textit{always} tolerable. 
Extending their technique to our situation is non-trivial, and could be a worthwhile direction for future research.
\end{itemize}

\section*{Acknowledgments}
DB and FB acknowledge support from the European Research Council (grant SEQUOIA 724063). 
DB was partly funded from the European
Union’s H2020 Framework Programme (H2020-MSCA-ITN-2014) under grant
agreement N\textsuperscript{\underline{o}}642685 MacSeNet.
DO was supported by the ERCIM Alain Bensoussan Fellowship.
We would like to thank Anatoli Juditsky for interesting discussions related to this work.

\bibliography{references}

\begin{thebibliography}{33}
\providecommand{\natexlab}[1]{#1}
\providecommand{\url}[1]{\texttt{#1}}
\expandafter\ifx\csname urlstyle\endcsname\relax
  \providecommand{\doi}[1]{doi: #1}\else
  \providecommand{\doi}{doi: \begingroup \urlstyle{rm}\Url}\fi

\bibitem[Beck and Teboulle(2003)]{beck2003mirror}
A.~Beck and M.~Teboulle.
\newblock Mirror descent and nonlinear projected subgradient methods for convex
  optimization.
\newblock \emph{Operations Research Letters}, 31\penalty0 (3):\penalty0
  167--175, 2003.

\bibitem[Blondel et~al.(2018)Blondel, Martins, and
  Niculae]{blondel2018learning}
M.~Blondel, A.~F.~T. Martins, and V.~Niculae.
\newblock Learning classifiers with {F}enchel-{Y}oung losses: Generalized
  entropies, margins, and algorithms.
\newblock \emph{arXiv preprint arXiv:1805.09717}, 2018.

\bibitem[B{\"u}hlmann and Van De~Geer(2011)]{buhlmann2011statistics}
P.~B{\"u}hlmann and S.~Van De~Geer.
\newblock \emph{Statistics for high-dimensional data: methods, theory and
  applications}.
\newblock Springer Science \& Business Media, 2011.

\bibitem[Chelba et~al.(2013)Chelba, Mikolov, Schuster, Ge, Brants, Koehn, and
  Robinson]{chelba2013one}
C.~Chelba, T.~Mikolov, M.~Schuster, Q.~Ge, T.~Brants, P.~Koehn, and
  T.~Robinson.
\newblock One billion word benchmark for measuring progress in statistical
  language modeling.
\newblock \emph{arXiv preprint arXiv:1312.3005}, 2013.

\bibitem[Clarkson et~al.(2012)Clarkson, Hazan, and
  Woodruff]{clarkson2012sublinear}
K.~L. Clarkson, E.~Hazan, and D.~P. Woodruff.
\newblock Sublinear optimization for machine learning.
\newblock \emph{Journal of the ACM (JACM)}, 59\penalty0 (5):\penalty0 23, 2012.

\bibitem[Defazio et~al.(2014)Defazio, Bach, and
  Lacoste-Julien]{defazio2014saga}
A.~Defazio, F.~Bach, and S.~Lacoste-Julien.
\newblock {SAGA}: A fast incremental gradient method with support for
  non-strongly convex composite objectives.
\newblock In \emph{Advances in neural information processing systems}, pages
  1646--1654, 2014.

\bibitem[Duchi et~al.(2010)Duchi, Shalev-Shwartz, Singer, and
  Tewari]{duchi2010composite}
J.~C. Duchi, S.~Shalev-Shwartz, Y.~Singer, and A.~Tewari.
\newblock Composite objective mirror descent.
\newblock In \emph{COLT}, pages 14--26, 2010.

\bibitem[Garber and Hazan(2011)]{garber2011approximating}
D.~Garber and E.~Hazan.
\newblock Approximating semidefinite programs in sublinear time.
\newblock In \emph{Advances in Neural Information Processing Systems}, pages
  1080--1088, 2011.

\bibitem[Garber and Hazan(2016)]{garber2016sublinear}
D.~Garber and E.~Hazan.
\newblock Sublinear time algorithms for approximate semidefinite programming.
\newblock \emph{Mathematical Programming}, 158\penalty0 (1-2):\penalty0
  329--361, 2016.

\bibitem[Grigoriadis and Khachiyan(1995)]{grigoriadis1995sublinear}
M.~D. Grigoriadis and L.~G. Khachiyan.
\newblock A sublinear-time randomized approximation algorithm for matrix games.
\newblock \emph{Operations Research Letters}, 18\penalty0 (2):\penalty0 53--58,
  1995.

\bibitem[Hazan et~al.(2011)Hazan, Koren, and Srebro]{hazan2011beatingsvm}
E.~Hazan, T.~Koren, and N.~Srebro.
\newblock Beating {SGD}: Learning {SVM}s in sublinear time.
\newblock In \emph{Advances in Neural Information Processing Systems}, pages
  1233--1241, 2011.

\bibitem[Johnson and Zhang(2013)]{johnson2013accelerating}
R.~Johnson and T.~Zhang.
\newblock Accelerating stochastic gradient descent using predictive variance
  reduction.
\newblock In \emph{Advances in Neural Information Processing Systems}, pages
  315--323, 2013.

\bibitem[Juditsky and Nemirovski(2011{\natexlab{a}})]{optbook1}
A.~Juditsky and A.~Nemirovski.
\newblock First-order methods for nonsmooth convex large-scale optimization,
  {I}: {G}eneral purpose methods.
\newblock \emph{Optimization for Machine Learning}, pages 121--148,
  2011{\natexlab{a}}.

\bibitem[Juditsky and Nemirovski(2011{\natexlab{b}})]{optbook2}
A.~Juditsky and A.~Nemirovski.
\newblock First order methods for nonsmooth convex large-scale optimization,
  ii: utilizing problems structure.
\newblock \emph{Optimization for Machine Learning}, pages 149--183,
  2011{\natexlab{b}}.

\bibitem[Kemperman(1969)]{kemperman1969optimum}
J.~H.~B. Kemperman.
\newblock On the optimum rate of transmitting information.
\newblock In \emph{Probability and information theory}, pages 126--169.
  Springer, 1969.

\bibitem[Laurent and Massart(2000)]{lama2000}
B.~Laurent and P.~Massart.
\newblock Adaptive estimation of a quadratic functional by model selection.
\newblock \emph{The Annals of Statistics}, 28\penalty0 (5):\penalty0
  1302--1338, 2000.

\bibitem[Mokhtari et~al.(2016)Mokhtari, Daneshmand, Lucchi, Hofmann, and
  Ribeiro]{mokhtari2016adaptive}
A.~Mokhtari, H.~Daneshmand, A.~Lucchi, T.~Hofmann, and A.~Ribeiro.
\newblock Adaptive {N}ewton method for empirical risk minimization to
  statistical accuracy.
\newblock In \emph{Advances in Neural Information Processing Systems}, pages
  4062--4070, 2016.

\bibitem[Nemirovski(2004)]{nemirovski2004prox}
A.~Nemirovski.
\newblock Prox-method with rate of convergence {$O(1/t)$} for variational
  inequalities with {L}ipschitz continuous monotone operators and smooth
  convex-concave saddle point problems.
\newblock \emph{SIAM Journal on Optimization}, 15\penalty0 (1):\penalty0
  229--251, 2004.

\bibitem[Nemirovski et~al.(2010)Nemirovski, Onn, and
  Rothblum]{nemirovski2010accuracy}
A.~Nemirovski, S.~Onn, and U.~G. Rothblum.
\newblock Accuracy certificates for computational problems with convex
  structure.
\newblock \emph{Mathematics of Operations Research}, 35\penalty0 (1):\penalty0
  52--78, 2010.

\bibitem[Nemirovsky and Yudin(1983)]{nemirovsky1983problem}
A.~Nemirovsky and D.~Yudin.
\newblock \emph{Problem complexity and method efficiency in optimization.}
\newblock Chichester, 1983.

\bibitem[Nesterov and Nemirovski(2013)]{nesterov2013first}
Y.~Nesterov and A.~Nemirovski.
\newblock On first-order algorithms for $l_1$/nuclear norm minimization.
\newblock \emph{Acta Numerica}, 22:\penalty0 509--575, 2013.

\bibitem[Ostrovskii and Harchaoui(2018)]{algorec2018arxiv}
D.~Ostrovskii and Z.~Harchaoui.
\newblock Efficient first-order algorithms for adaptive signal denoising.
\newblock In \emph{Proceedings of the 35th ICML conference}, volume~80, pages
  3946--3955, 2018.

\bibitem[Palaniappan and Bach(2016)]{palaniappan2016stochastic}
B.~Palaniappan and F.~Bach.
\newblock Stochastic variance reduction methods for saddle-point problems.
\newblock In \emph{Advances in Neural Information Processing Systems}, pages
  1416--1424, 2016.

\bibitem[Partalas et~al.(2015)Partalas, Kosmopoulos, Baskiotis, Artieres,
  Paliouras, Gaussier, Androutsopoulos, Amini, and Galinari]{partalas2015lshtc}
I.~Partalas, A.~Kosmopoulos, N.~Baskiotis, T.~Artieres, G.~Paliouras,
  E.~Gaussier, I.~Androutsopoulos, M.-R. Amini, and P.~Galinari.
\newblock {LSHTC:} {A} benchmark for large-scale text classification.
\newblock \emph{arXiv preprint arXiv:1503.08581}, 2015.

\bibitem[Schmidt et~al.(2017)Schmidt, Le~Roux, and Bach]{schmidt2017minimizing}
M.~Schmidt, N.~Le~Roux, and F.~Bach.
\newblock Minimizing finite sums with the stochastic average gradient.
\newblock \emph{Mathematical Programming}, 162\penalty0 (1-2):\penalty0
  83--112, 2017.

\bibitem[Shalev-Shwartz and Ben-David(2014)]{shalev2014understanding}
S.~Shalev-Shwartz and S.~Ben-David.
\newblock \emph{Understanding machine learning: From theory to algorithms}.
\newblock Cambridge university press, 2014.

\bibitem[Shalev-Shwartz and Zhang(2013)]{shalev2013stochastic}
S.~Shalev-Shwartz and T.~Zhang.
\newblock Stochastic dual coordinate ascent methods for regularized loss
  minimization.
\newblock \emph{Journal of Machine Learning Research}, 14\penalty0
  (Feb):\penalty0 567--599, 2013.

\bibitem[Shalev-Shwartz et~al.(2011)Shalev-Shwartz, Singer, Srebro, and
  Cotter]{shalev2011pegasos}
S.~Shalev-Shwartz, Y.~Singer, N.~Srebro, and A.~Cotter.
\newblock Pegasos: {P}rimal estimated sub-gradient solver for {SVM}.
\newblock \emph{Mathematical programming}, 127\penalty0 (1):\penalty0 3--30,
  2011.

\bibitem[Shi et~al.(2017)Shi, Zhang, and Yu]{shi2017bregman}
Z.~Shi, X.~Zhang, and Y.~Yu.
\newblock Bregman divergence for stochastic variance reduction: saddle-point
  and adversarial prediction.
\newblock In \emph{Advances in Neural Information Processing Systems}, pages
  6031--6041, 2017.

\bibitem[Sion(1958)]{sion1958}
M.~Sion.
\newblock On general minimax theorems.
\newblock \emph{Pacific Journal of Mathematics}, 8\penalty0 (1):\penalty0
  171--176, 1958.

\bibitem[Sra(2012)]{sra2012fast}
S.~Sra.
\newblock Fast projections onto mixed-norm balls with applications.
\newblock \emph{Data Mining and Knowledge Discovery}, 25\penalty0 (2):\penalty0
  358--377, 2012.

\bibitem[Xiao et~al.(2017)Xiao, Yu, Lin, and Chen]{xiao2017dscovr}
L.~Xiao, A.~W. Yu, Q.~Lin, and W.~Chen.
\newblock {DSCOVR}: Randomized primal-dual block coordinate algorithms for
  asynchronous distributed optimization.
\newblock \emph{arXiv preprint arXiv:1710.05080}, 2017.

\bibitem[Yu(2013)]{yu2013strong}
Y.-L. Yu.
\newblock The strong convexity of von {N}eumann{'}s entropy.
\newblock \emph{Unpublished note}, June 2013.
\newblock URL \url{http://www.cs.cmu.edu/~yaoliang/mynotes/sc.pdf}.

\end{thebibliography}
\bibliographystyle{plainnat}

\clearpage

\appendix

\section{Motivation for Multiclass Hinge Loss}
\label{SVMMotivation}

We justify the multiclass extension~\eqref{eq:loss-hinge} of the hinge loss due to~\cite{shalev2014understanding}.
In the binary case, the hinge loss is
\[
\frac{1}{n} \sum_{i=1}^n \left[ \max(0, 1-\wt y_i u^\top x_i) \right],
\]
where~$u \in \R^d$ and $\wt y_i \in \{-1,1\}$. 
Introducing~$y = e_{\wt y} \in \{e_{-1}, e_{1}\}$ where~$e_j$ is the~$j$-th standard basis vector (the dimensions of space are symbolically indexed in~$\{-1,1\}$), and putting~$u_1 = -u_{-1} = \frac{u}{2}$, we can rewrite the loss as
\[
\max(0, 1-\wt y u^\top x) = \max_{k \in \{1,-1\}} \left\{ \ind\{e_k \ne y\} + u_{k}^\top  x  - u_{\wt y}^\top x \right\}.
\]
The advantage of this reformulation is that we can naturally pass to the multiclass case, by replacing the set $\{-1,1\}$ with~$\{1,...,K\}$ and introducing~$u_1, ..., u_K \in \R^d$ without any restrictions: 
\[
\begin{aligned}
\max_{k \in \{1,..., K\}} \left\{ \ind\{e_k \ne y\} + u_{k}^\top  x  - u_{\wt y}^\top x \right\} 
&= \max_{v \in \{e_1,..., e_K\}} \left\{ \ind\{v \ne y\} + \sum_{l=1}^K (v[l] - y[l]) u_l^\top  x\right\} \\
&= \max_{v \in \{e_1,..., e_K\}} \left\{ \ind\{v \ne y\} + (v - y)^\top U^\top  x\right\} =: \ell(U,(x,y)),
\end{aligned}
\]
where~$a[l]$ denotes the~$l$-th element of a column-vector~$a$, and~$U \in \R^{d \times K}$ has $u_l$ as its $l$-th column. 
Finally, we can rewrite~$\ell(U,(x,y))$ as follows:
\[
\ell(U,(x,y)) = \max_{v \in \Delta_K} \left\{1 - v^\top y + (v - y)^\top U^\top  x\right\}.
\]
This is because we maximize an affine function of~$v$, and~$1 - v^\top y = \ind\{v \ne y\}$ at the vertices. Thus, we obtain the Fechel dual representation of the multiclass hinge loss.
Adding the regularization term~$\|\cdot\|_{\Unorm}$, we also arrive at the saddle-point problem
\[
\min_{ U \in \rb^{d \times k} } \max_{V\in \Delta_k^{\otimes n}}
1 - \frac{1}{n}\tr[V^\top Y] + 
\frac{1}{n} \tr \big[
( V - Y )^\top X U
\big]  + \lambda \| U \|_{\Unorm}.
\]

\section{General Accuracy Bounds for Composite Saddle-Point Mirror Descent}
\label{sec:accuracy-general}

\paragraph{Deterministic Case.}
Here we provide general accuracy bounds which are instantiated in Theorems~\ref{th:md-deterministic} and~\ref{th:md-partial}. 
Below we outline the general setting that encompasses, in particular, the case of~\eqref{SaddleUhat} solved via~\eqref{NonStochasticUpdates} with initialization~\eqref{eq:init-full}.

\begin{itemize}
\item
We consider a convex-concave saddle-point problem
\[
\min_{U \in \UU} \max_{V \in \VV} f(U,V)
\]
with a composite objective
\[
f(U,V) = \Phi(U,V-Y) + \Upsilon(U) - \FF(V),
\]
where
\[
\Phi(U,V) = \frac{1}{n}V^\top X U
\] 
is a bilinear function, and~$\Upsilon(U),\FF(V)$ are convex ``simple'' terms. 
Moreover, we assume that the primal feasible set~$\UU$ belongs to the~$\|\cdot\|_{\Unorm}$-norm ball with radius~$R_{\ast}$, the dual constraint set~$\VV$ belongs to the~$\|\cdot\|_\Vnorm$-norm ball with radius~$R_{\VV}$, and~$\| Y \|_{\Vnorm} \leqs R_{\VV}$.\footnotemark~ 
To simplify the results, we make the assumption (satisfied in all known to us situations):
\begin{equation}
\label{eq:omegas-vs-radii-lower}
\Omega_{\UU} \geqs R_{\ast}^2, \quad \Omega_{\VV} \geqs R_{\VV}^2.
\end{equation}
\item
Recall that the vector field of partial gradients of~$\Psi(U,V) := \Phi(U,V-Y)$ is
\begin{equation}
\label{eq:gradient-field-proof}
\begin{aligned}
G(W)
:&= (\nabla_{U} \Psi(U,V), -\nabla_{V} \Psi(U,V)) \\
&= \frac{1}{n} (X^\top(V-Y),-XU)
\end{aligned}
\end{equation}
\footnotetext{Note that the linear term~$\frac{1}{n}Y^\top X U$ can be absorbed into the simple term~$\Upsilon(U)$, which will slightly improve the bound in Theorem~\ref{th:md-general}. However, this improvement is impossible in the stochastic version of the algorithm where we sample the linear form~$Y^\top X$ but not the gradient of~$\Upsilon(U)$.}
\item
Given the partial proximal setups~$(\|\cdot\|_{\Unorm},\phi_{\UU}(\cdot))$ and~$(\|\cdot\|_{\Vnorm},\phi_{\VV}(\cdot))$, we run Composite Mirror Descent~\eqref{MirrorDescent} on the vector field~$G(W)$ with the joint penalty term~\[
h(W) = \Upsilon(U) + \FF(V),
\] 
the ``balanced'' joint potential given by~\eqref{eq:joint-potential}, and stepsizes~$\gamma_t$.
\end{itemize}

We now provide the convergence analysis of Mirror Descent, extending the argument of Lem.~1 in~\cite{duchi2010composite} to composite saddle-point optimization.

\begin{theorem}
\label{th:md-general}
In the above setting, let~$(\bar{U}^T, \bar{V}^T) = \frac{1}{T}\sum_{t=0}^{T-1} ({U}^t,V^t)$ be the average of the first~$T$ iterates of the composite Mirror Descent~\eqref{MirrorDescent} with constant stepsize
\[
\gamma_t \equiv \frac{1}{\cL_{\Unorm,\Vnorm}\sqrt{5T\Omega_{\UU}\Omega_\VV}},
\]
where
\[
\cL_{\Unorm,\Vnorm} := \frac{1}{n}\sup_{\| U \|_{\Unorm} \le 1} \|XU\|_{\Vnorm^*}.
\]
Then we have the following guarantee for the duality gap of~$f(\cdot)$:
\[
\Gap(\bar U^T, \bar V^T) \leqs \frac{2\sqrt{5}\cL_{\Unorm,\Vnorm}\sqrt{\Omega_{\UU}\Omega_\VV}}{\sqrt{T}} 
+ \frac{\Upsilon(U^0) - \min_{U \in \UU} \Upsilon(U)}{T} + \frac{\FF(V^0) - \min_{V \in \VV} \FF(V)}{T}.
\]
Moreover, if one of the functions~$\Upsilon(U)$,~$\FF(V)$ is affine, the corresponding~$O(1/T)$ error term vanishes from the bound.
\end{theorem}

\begin{proof}

$\boldsymbol{1^o}.$
We begin by introducing the norm for~$W = (U,V)$:
\begin{equation}
\label{eq:joint-norm}
\|W\|_{\Wnorm} = \sqrt{\frac{\|U\|_\Unorm^2}{2\Omega_\UU}+\frac{\|V\|_\Vnorm^2}{2\Omega_\VV}},
\end{equation}
and its dual norm defined for~$G = (G_U,G_V)$ with~$G_U \in \R^{d \times k}$ and~$G_V \in \R^{n \times k}$:
\begin{equation}
\label{eq:joint-norm-dual}
\|G\|_{\Wnorm^*} = \sqrt{2\Omega_\UU\|G_U\|_{\Unorm^*}^2 + 2\Omega_\VV\|G_V\|_{\Vnorm^*}^2},
\end{equation}
where~$\|\cdot\|_{\Unorm^*}$ and~$\|\cdot\|_{\Vnorm^*}$ are the dual norms for~$\|\cdot\|_{\Unorm}$ and~$\|\cdot\|_{\Vnorm}$ correspondingly.
We now make a few observations. First, the joint potential~$\phi_{\WW}(W)$ given by~\eqref{eq:joint-potential} is~$1$-strongly convex with respect to the norm~$\|\cdot\|_{\Wnorm}$.
Second, we can compute the potential difference corresponding to~$\phi_{\WW}$:
\begin{equation}
\label{eq:md-potential-difference-estimate}
\Omega_{\WW} := \max_{W \in \WW} \phi_{\WW}(W) - \min_{W \in \WW} \phi_{\WW}(W) = 1
\end{equation}
Finally, by~\eqref{eq:gradient-field-proof} and~\eqref{eq:cross-lip-bilinear} we have
\[
\max_{W \in \WW} \|G_U(W)\|_{\Unorm^*} \leqs 2\cL_{\Unorm,\Vnorm}R_{\VV}, \quad \max_{W \in \WW} \|G_V(W)\|_{\Vnorm^*} \leqs \cL_{\Unorm,\Vnorm}R_{\ast},
\]
combining which with~\eqref{eq:omegas-vs-radii-lower} we bound the~$\|\cdot\|_{\Wnorm^*}$-norm of~$G(W)$ on~$\WW$:
\begin{equation}
\label{eq:md-gradient-norm}
\begin{aligned}
\max_{W \in \WW} \| G(W) \|_{\Wnorm^*} 
&\leqs \sqrt{10}\cL_{\Unorm,\Vnorm} \sqrt{\Omega_{\UU}\Omega_\VV}.
\end{aligned}
\end{equation}

$\boldsymbol{2^o}.$
We now follow the convergence analysis of composite Mirror Descent, see~\cite{duchi2010composite}, extending it to convex-concave objectives.
By the convexity properties of~$\Psi(U,V) = \Phi(U,V-Y)$,~for any~$(\bar U, \bar V) \in \WW$ and~$(U,V) \in \WW$ it holds
\[
\begin{aligned}
\Psi(\bar U,V) - \Psi(U,\bar V) 
&= \Psi(\bar U,V) - \Psi(\bar U,\bar V) + \Psi(\bar U,\bar V) - \Psi(U,\bar V) \\
&\leqs \lang \nabla_U \Psi(\bar U,\bar V), \bar U - U\rang - \lang \nabla_V \Psi(\bar U,\bar V), \bar V - V\rang \\
& = \lang G(\bar W), \bar W - W \rang,
\end{aligned}
\]
Let~$W^t = (U^t, V^t)$ be the~$t$-th iterate of~\eqref{MirrorDescent} for~$t \ge 1$.
By convexity of~$\Upsilon(U)$ and~$\FF(V)$, and denoting~$h(W) = \Upsilon(U) + \FF(V)$, we have, for any~$W = [U,V]$, that
\begin{equation}
\label{eq:md-gap-main}
\begin{aligned}
\Psi(U^{t-1},V) - \Psi(U,V^{t-1}) &+ h(W^{t}) - h(W) \\
&\leqs \lang G(W^{t-1}), W^{t-1} - W\rang + \lang \partial h(W^t), W^t - W \rang.
\end{aligned}
\end{equation}
Let us now bound the right-hand side. 
Note that the first-order optimality condition for~\eqref{MirrorDescent} (denoting~$\phi(\cdot) := \phi_{\WW}(\cdot)$ the joint potential) writes\footnotemark
\begin{equation}
\label{eq:first-order-general}
\lang \gamma_t [G(W^{t-1}) + \partial h(W^{t})] + \nabla \phi(W^{t}) - \nabla \phi(W^{t-1}), W^{t} - W \rang \leqs 0.
\end{equation}
Combining this with~\eqref{eq:md-gap-main}, we get
\begin{equation}
\label{eq:md-one-term-prelim}
\begin{aligned}
\gamma_t [\Psi(U^{t-1},V) - \Psi(U,V^{t-1}) + h(W^{t}) - h(W)] 
&\leqs \lang \nabla \phi(W^{t-1}) - \nabla \phi(W^{t}), W^{t} - W \rang \\
&+ \gamma_t \lang G(W^{t-1}), W^{t-1} - W^t \rang.
\end{aligned}
\end{equation}
By the well-known identity,
\begin{align}
\label{eq:bregman-3points}
\lang \nabla \phi(W^{t-1}) - \nabla \phi(W^{t}), W^{t} - W \rang 
&= D_{\phi}(W,W^{t-1}) - D_{\phi}(W,W^t) - D_{\phi}(W^t,W^{t-1}),
\end{align}
see,~e.g.,~\citet[Lemma~4.1]{beck2003mirror}. 
On the other hand, by the Fenchel-Young inequality we have
\begin{equation}
\label{eq:md-young}
\begin{aligned}
\gamma_t \lang G(W^{t-1}), W^{t-1} - W^t \rang 
&\leqs \frac{\gamma_t^2 \| G(W^{t-1})\|_{\Wnorm^*}^2}{2} + \frac{\| W^{t-1} - W^t \|_{\Wnorm}^2}{2}  \\
&\leqs 5\gamma_t^2 \cL_{\Unorm,\Vnorm}^2 \Omega_{\UU}\Omega_\VV + D_{\phi}(W^t,W^{t-1}),
\end{aligned}
\end{equation}
where we used~\eqref{eq:md-gradient-norm} and~1-strong convexity of~$\phi(\cdot)$ with respect to~$\|\cdot\|_{\Wnorm}$. 
Thus, we obtain
\begin{equation}
\label{eq:gap-final-bound}
\begin{aligned}
\gamma_t [\Psi(U^{t-1},V) - \Psi(U,V^{t-1}) + h(W^{t}) - h(W)] \le &D_{\phi}(W,W^{t-1}) - D_{\phi}(W,W^t) + 5\gamma_t^2 \cL_{\Unorm,\Vnorm}^2 \Omega_{\UU}\Omega_\VV.
\end{aligned}
\end{equation}

$\boldsymbol{3^o}.$
Now, assuming the constant stepsize, by the convexity properties of~$\Psi(\cdot,\cdot)$ and~$h(\cdot)$ we obtain
\begin{equation}
\label{eq:gap-averaging}
\begin{aligned}
f(\bar U^T,V) - f(U,\bar V^T) 
&= \Psi(\bar U^{T},V) - \Psi(U,\bar V^{T}) + h(\bar W^{T}) - h(W) \\
&\leqs \frac{1}{T} \sum_{t=1}^{T} \Psi(U^{t-1},V) - \Psi(U,V^{t-1}) + h(W^{t-1}) - h(W) \\
&\leqs \frac{1}{T} \left( h(W^{0}) - h(W^{T}) + \frac{D_{\phi}(W,W^0)}{\gamma} + 5 T \gamma   \cL_{\Unorm,\Vnorm}^2 \Omega_{\UU}\Omega_\VV\right) \\
&\leqs \frac{1}{T} \left( h(W^{0}) - \min_{W \in \WW} h(W) + \frac{1}{\gamma} + 5 T \gamma   \cL_{\Unorm,\Vnorm}^2 \Omega_{\UU}\Omega_\VV\right).
\end{aligned}
\end{equation}
where for the third line we substituted~\eqref{eq:gap-final-bound}, simplified the telescoping sum, and used that~$D(W,W^T) \geqs 0$, and in the last line we used~$D(W,W^0) \leqs \Omega_{\WW} \leqs 1$, cf.~\eqref{eq:md-potential-difference-estimate}.
The choice
\[
\gamma = \frac{1}{\cL_{\Unorm,\Vnorm}\sqrt{5T\Omega_{\UU}\Omega_\VV}}, 
\]
results in the accuracy bound from the premise of the theorem:
\[
\Gap(\bar U^T, \bar V^T) \leqs \frac{2\sqrt{5}\cL_{\Unorm,\Vnorm}\sqrt{\Omega_{\UU}\Omega_\VV}}{\sqrt{T}} + \frac{h(W^0) - \min_{W \in \WW} h(W)}{T}.
\]
Finally, assume that one of the terms~$\Upsilon(U)$,~$\FF(V)$ is affine -- w.l.o.g. let it be~$\Upsilon(U)$. Then, since~$\nabla \Upsilon(U)$ is constant,~$\partial h(W^{t}) = (\nabla \Upsilon(U^t), \partial \FF(V^t))$ in~\eqref{eq:first-order-general} can be replaced with~$(\nabla \Upsilon(U^{t-1}), \partial \FF(V^t))$. 
Then in~\eqref{eq:gap-final-bound} we can replace~$h(W^t) - h(W^0)$ with~$\Upsilon(U^{t-1}) - \Upsilon(U) + \FF(V^{t}) - \FF(V)$, implying that the term~$h(W^0) - h(W^{T})$ in the right-hand side of~\eqref{eq:gap-averaging} gets replaced with~$\FF(V^0) - \FF(V^t)$. The claim is proved.
\end{proof}
\footnotetext{Note that~$\phi(W)$ is continuously differentiable in the interior of~$\WW$, and~$\nabla \phi$ diverges on the boundary of~$\WW$, then the iterates are guaranteed to stay in the interior of~$\WW$~\cite{beck2003mirror}.}

\paragraph{Stochastic Mirror Descent.} 
We now consider the stochastic setting that allows to encompass~\eqref{StochasticUpdates}. Stochastic Mirror Descent is given by
\begin{equation}
\label{StochasticMirrorDescent}
\begin{aligned}
W^0 &= \min_{W \in \WW} \phi_{\WW}(W); \\
W^{t} &= \arg\min\limits_{W\in\WW}\left\{h(W) + \langle \Xi(W^{t-1}), W \rangle + \frac{1}{\gamma_t}D_{\phi_{\mathcal{W}}}(W,W^{t-1})\right\}, \; t \ge 1,
\end{aligned}
\end{equation}
where
\[
\Xi(W) := \frac{1}{n}(\eta_{V,Y},-\xi_{U})
\]
is the unbiased estimate of the first-order oracle~$G(W) = \frac{1}{n}(X^\top(V-Y),-XU)$.
Let us introduce the corresponding variance proxies (refer to the preamble of Section~\ref{sec:sampling} for the discussion):
\begin{equation}
\label{def:variance-proxies-general}
\sigma^2_{\UU} = \frac{1}{n^2}\sup_{U \in \UU} \E\left[\left\|XU - \xi_{U} \right\|_{\Vnorm^*}^2\right], \quad
\sigma^2_{\VV} = \frac{1}{n^2}\sup_{(V,Y) \in \VV \times \VV} \E\left[\left\| X^{\top}(V-Y) - \eta_{V,Y} \right\|_{\Unorm^*}^2\right].
\end{equation}
We assume that the noises~$G(W^{t-1}) - \Xi(W^{t-1})$ are independent along the iterations of~\eqref{StochasticMirrorDescent}. 
In this setting, we prove the following generalization of Theorem~\ref{th:md-general}:

\begin{theorem}
\label{th:md-general-stochastic}
Let~$(\bar{U}^T, \bar{V}^T) = \frac{1}{T}\sum_{t=0}^{T-1} ({U}^t,V^t)$ be the average of the first~$T$ iterates of stochastic composite mirror descent~\eqref{StochasticMirrorDescent} with constant stepsize
\[
\gamma_t \equiv \frac{1}{\sqrt{T}} \min \left\{ \frac{1}{\sqrt{10}\cL_{\Unorm,\Vnorm}\sqrt{\Omega_{\UU}\Omega_\VV}}, \; \frac{1}{\sqrt{2}\sqrt{\Omega_{\UU} \bar\sigma_{\vphantom{\UU} \VV}^2 + \Omega_{\vphantom{\UU}\VV} \bar\sigma_{\UU}^2}}\right\},
\]
where~$\cL_{\Unorm,\Vnorm},\Omega_{\UU},\Omega_\VV$ are the same as in Theorem~\ref{th:md-general-stochastic}, and~$\bar \sigma_{\UU}^2, \bar \sigma_{\VV}^2$ are the upper bounds for~$\sigma_{\UU}^2, \sigma_{\VV}^2$, cf.~\eqref{def:variance-proxies-general}. 
Then it holds
\[
\begin{aligned}
\E[\Gap(\bar U^T, \bar V^T)] 
&\leqs \frac{2\sqrt{10}\cL_{\Unorm,\Vnorm}\sqrt{\Omega_{\UU}\Omega_\VV}}{\sqrt{T}} 
+ \frac{2\sqrt{2}\sqrt{\Omega_{\UU} \bar\sigma_{\vphantom{\UU} \VV}^2 + \Omega_{\vphantom{\UU}\VV} \bar\sigma_{\UU}^2}}{\sqrt{T}} \\
&\;\;+ \frac{\Upsilon(U^0) - \min_{U \in \UU} \Upsilon(U)}{T} + \frac{\FF(V^0) - \min_{V \in \VV} \FF(V)}{T},
\end{aligned}
\]
where~$\E[\cdot]$ is the expectation over the randomness in~\eqref{StochasticMirrorDescent}.
Moreover, if one of the functions~$\Upsilon(U)$,~$\FF(V)$ is affine, the corresponding~$O(1/T)$ term can be discarded.
\end{theorem}

\begin{proof}
The proof closely follows that of Theorem~\ref{th:md-general}. 
First,~$\boldsymbol{1^o}$ remains unchanged. Then, in the first-order condition~\eqref{eq:first-order-general} one must replace~$G(W^{t-1})$ with~$\Xi(W^{t-1})$, which results in replacing~\eqref{eq:md-one-term-prelim} with
\begin{equation*}
\begin{aligned}
\gamma_t &[\Psi(U^{t-1},V) - \Psi(U,V^{t-1}) + h(W^{t}) - h(W)] \\
&\leqs \lang \nabla \phi(W^{t-1}) - \nabla \phi(W^{t}), W^{t} - W \rang \\
&+ \gamma_t \lang \Xi(W^{t-1}), W^{t-1} - W^t \rang \\
&+ \gamma_t \lang G(W^{t-1}) - \Xi(W^{t-1}), W^{t-1} - W \rang,
\end{aligned}
\end{equation*}
where the last term has zero mean. 
The term~$\gamma_t \lang \Xi(W^{t-1}), W^{t-1} - W^t \rang$ can be bounded using Young's inequality, and~$1$-strong convexity of~$\phi(\cdot)$, cf.~\eqref{eq:md-young}:
\begin{equation*}
\begin{aligned}
\gamma_t &\lang \Xi(W^{t-1}), W^{t-1} - W^t \rang \\
&\leqs \frac{\gamma_t^2 \| \Xi(W^{t-1})\|_{\Wnorm^*}^2}{2} + \frac{\| W^{t-1} - W^t \|_{\Wnorm}^2}{2} \\
&\leqs \gamma_t^2 \| G(W^{t-1})\|_{\Wnorm^*}^2 + \gamma_t^2\| \Xi(W^{t-1}) - G(W^{t-1})\|_{\Wnorm^*}^2 + D_{\phi}(W^t,W^{t-1}).
\end{aligned}
\end{equation*}
Combining~\eqref{eq:md-gradient-norm},~\eqref{eq:joint-norm-dual}, and~\eqref{def:variance-proxies-general}, this implies
\begin{equation*}
\begin{aligned}
\gamma_t &\lang \Xi(W^{t-1}), W^{t-1} - W^t \rang 
\leqs 10\gamma_t^2 \cL_{\Unorm,\Vnorm}^2 \Omega_{\UU}\Omega_\VV + 2\gamma_t^2 \left(\Omega_{\UU} \bar\sigma_{\VV}^2 + \Omega_{\VV}  \bar\sigma_{\UU}^2 \right)+ D_{\phi}(W^t,W^{t-1}).
\end{aligned}
\end{equation*}
Using~\eqref{eq:bregman-3points} and~\eqref{eq:md-potential-difference-estimate}, this results in
\begin{equation*}
\begin{aligned}
\E[\Gap(\bar U^T, \bar V^T)]
&= \E\left[\max_{(U,V) \in \WW} \left\{ f(\bar U^T,V) - f(U,\bar V^T) \right\}\right]\\
&\leqs \frac{1}{T} \left[ h(W^{0}) - \min_{W \in \WW} h(W) + \frac{1}{\gamma} + 2T \gamma  \left( 5\cL_{\Unorm,\Vnorm}^2 \Omega_{\UU}\Omega_\VV +  \Omega_{\UU} \bar\sigma_{\VV}^2 + \Omega_{\VV} \bar\sigma_{\UU}^2 \right) \right];
\end{aligned}
\end{equation*}
note that maximization on the left is \textit{under} the expectation (and not vice versa) because the right hand side is independent from~$W=(U,V)$. 
Choosing~$\gamma$ to balance the terms, we arrive at the desired bound. Finally, improvement in the case of affine~$\Upsilon(U)$,~$\FF(V)$ is obtained in the same way as in Theorem~\ref{th:md-general-stochastic}.
\end{proof}

\section{Auxiliary Lemmas}
\label{sec:lemmas}

\begin{lemma} 
\label{lemmaForU}
Let $X^0\in \R^n_+$ and 
$X^1 = \arg\min\limits_{\substack{\|X\|_1\leqslant R\\
X\in \R_+^n}} \bigg\{C_1\|X\|_1 + \langle S, X\rangle + C_2 \sum\limits_{i=1}^n X_i\log\dfrac{X_i}{X_i^0}  \bigg\}$. 
Then,
\[
X^1_i =  \rho\cdot \dfrac{ X_i^0\cdot \exp(-S_i/C_2)}{M},
\] 
where $M = \sum\limits_{j=1}^n X_j^0\cdot \exp(-S_j/C_2)$ and $\rho = \min(M\cdot e^{-\frac{C_1+C_2}{C_2}}, R)$. 
\end{lemma}

\begin{proof}
Clearly, we have
\[
X^1 = \arg\min\limits_{r \leqs R}\min\limits_{\substack{\|X\|_1 = r\\
X\in \R_+^n}} \bigg\{ C_1 r + \langle S, X\rangle + C_2 \sum\limits_{i=1}^n X_i\log\dfrac{X_i}{X_i^0}  \bigg\}.
\]
Let us first do the internal minimization. By simple algebra, the first-order optimality condition for the Lagrangian dual problem (with constraint~$\|X\|_1 = R$) amounts to
\[
S_i + C_2 + C_2\log X_i^{1}- C_2\log X_i^0 + \kappa = 0,
\]
where~$\kappa$ is Lagrange multiplier, and~$\sum\limits_{i=1}^nX_i^{1} = r$. 
Equivalently,
\[
X_i^{1} = X_i^0 \cdot \exp \left(-\left(\frac{\kappa+C_2+S_i}{C_2}\right)\right),
\]
that is,
\[
X_i^{1} = r\cdot \frac{ X_i^0\cdot \exp(-S_i/C_2)}{\sum\limits_j X_j^0\cdot \exp(-S_j/C_2)}.
\]
Denoting~$D_j = \exp(-S_j/C_2)$ and~$M = \sum\limits_j X_j^0\cdot D_j$ and substituting for~$X_1$ in the external minimization problem, we arrive at
\[
\rho = \arg\min\limits_{r \leqslant R}\left\{ C_1 r + r \sum\limits_i\frac{X_i^0D_i S_i}{M} + C_2\cdot r \sum\limits_i \frac{X_i^0D_i }{M}\cdot \log\left[\frac{r \cdot D_i}{M} \right] \right\}.
\]
One can easily verify that the counterpart of this minimization problem with~$R = \infty$ has a unique stationary point~$r^* = M\cdot e^{-\frac{C_1+C_2}{C_2}} > 0$.
As the minimized function is convex, the minimum is attained at the point $\rho = \min(r^*, R)$.
\end{proof}

\begin{lemma}
\label{lemma_variance_U}
Given~$X \in \R^{n \times d}$ and mixed norms (cf.~\eqref{def:mixed-norms})~$\| \cdot \|_{\puone \times \putwo}$ on~$\R^{d \times k}$ and~$\| \cdot \|_{\qvone \times \qvtwo}$ on~$\R^{n \times k}$ with~$\putwo \leqslant \qvtwo$, one has
\[
\sup_{\|U\|_{\puone\times \putwo}\leqslant 1}\left\{\sum\limits_{i=1}^d\| X(:,i) \|_{\qvone} \cdot \| {U}(i,:) \|_{\qvtwo} \right\} = \| X^\top\|_{\quone\times \qvone},
\]
where~$q_U^{1}$ is the conjugate of~$p_U^{1}$, i.e.,~$1/p_U^{1} + 1/q_U^{1} = 1$.
\end{lemma}

\begin{proof}
First assume~$\putwo = \qvtwo$. Let~$a_i = \|{X}(:,i)\|_{\qvone}$,~$u_i = \|{U}(i,:)\|_{\qvtwo}$,~$1 \leqs i \leqs d$. Then,
\[
\sup_{\|U\|_{\puone\times \putwo}\leqslant 1}\left\{\sum\limits_{i=1}^d\| X(:,i) \|_{\qvone} \cdot \| {U}(i,:) \|_{\qvtwo} \right\} = 
\sup\limits_{\| u \|_{\puone} \leqslant 1} \sum\limits_{i=1}^d a_i u_i = \|a\|_{\quone} = \|{X}^\top\|_{\quone\times\qvone}. 
\]
Now let~$\putwo < \qvtwo$. Then, for any~$i \le d$ one has~$\|U(i,:)\|_{\qvtwo} < \|U(i,:)\|_{\putwo}$ unless $U(i,:)$ has a single non-zero element, in which case~$\|U(i,:)\|_{\qvtwo} = \|U(i,:)\|_{\putwo}$. Hence, the supremum must be attained on such~$U$, for which the previous argument applies.
\end{proof}

\begin{lemma}
\label{lemma_variance_V}
In the setting of Lemma~\ref{lemma_variance_U}, for any~$\quone \geqslant 1$ and~$\qutwo \geqslant 1$ one has:
\[
\sup_{\| V \|_{\infty \times 1} \leqs 1}\left\{\sum\limits_{i=1}^n\| {X}^\top(:,i) \|_{\quone} \cdot \| V(i,:) \|_{\qutwo} \right\} = \|{X}\|_{1\times \quone}.
\]
\end{lemma}

\begin{proof}
The claim follows by instatiating Lemma~\ref{lemma_variance_U}.
\end{proof}

\section{Deferred Proofs}

\subsection{Proof of Proposition~\ref{th:lip}}
\label{sec:proof-lip}

By~\eqref{eq:cross-lip-bilinear}, and verifying that the dual norm to~$\|\cdot\|_{2 \times 1}$ is~$\|\cdot\|_{2 \times \infty}$, we have 
\[
\cL_{\Unorm,\Vnorm} = \sup_{\|U\|_{1\times 1}\leqslant 1}\|XU \|_{2\times \infty}. 
\]
The maximization over the unit ball~$\|U\|_{1\times 1}\leqslant 1$ can be replaced with that over its extremal points, which are the matrices~$U$ that have zeroes in all positions except for one in which there is~$1$. 
Let~$(i, j)$ be this position, then for every such~$U$ we have:
\[
\|XU\|_{2 \times \infty} = \sqrt{\sum_{l =1}^n \sup_j |X(l,:) U(:,j)|^2} = \sqrt{\sum_{l =1}^n |X(l,i)|^2} = \sqrt{\sum_{l =1}^n |X^\top(i,l)|^2}.
\]
As a result,
\[
\sup\limits_{ \|U\|_{1\times 1} \leqs 1} \|XU\|_{2 \times \infty} = \sup_{1 \le i \le k}\sqrt{\sum_{l =1}^n |X^\top(i,l)|^2} = \left\|X^\top\right\|_{\infty \times 2}. 
\qed
\]

\subsection{Proof of Proposition~\ref{th:variance-partial}}
\label{sec:proof-variance-partial}

We prove an extended result that holds when~$\|\cdot\|_{\Unorm}$ and~$\|\cdot\|_{\Vnorm}$ are more general mixed~$(\ell_p \times \ell_q)$-norms, cf.~\eqref{def:mixed-norms}. 

\begin{proposition}
\label{th:variance-partial-general}
Let~$\|\cdot\|_{\Unorm} = \|\cdot\|_{\puone\times \putwo}$ on~$\R^{2d \times k}$, and~$\|\cdot\|_{\Vnorm} = \|\cdot\|_{\pvone \times \pvtwo}$ on~$\R^{n \times k}$.
Then, optimal solutions~$p^{\ast} = p^\ast(\wh X, U)$ and~$q^\ast = q^\ast(\wh X,V,Y)$ to~\eqref{eq:optimal-prob} are given by
\[
p^{\ast}_i = \frac{\| \wh{X}(:,i) \|_{\qvone} \cdot  \| {U}(i,:) \|_{\qvtwo}}{\sum_{\imath=1}^{2d} \| \wh{X}(:,\imath) \|_{\qvone} \cdot  \| {U}(\imath,:) \|_{\qvtwo}}, \quad
q^{\ast}_j = \frac{\| \wh{X}(j,:) \|_{\quone}\cdot  \| V(j,:) - Y(j,:)\|_{\qutwo} }{\sum_{\jmath=1}^n \Vert \wh{X}(\jmath,:)\Vert_{\quone} \cdot  \Vert V(\jmath,:) - Y(\jmath,:)\Vert_{\qutwo}}.
\]
Moreover, we can bound their respective variance proxies (cf.~\eqref{def:variance-proxies}): introducing
\[
\cL_{\Unorm,\Vnorm}^2 = \frac{1}{n^2}\sup_{\| U \|_{\Unorm} \le 1}\Vert \wh{X} {U}\Vert_{\Vnorm^\ast}^2,
\]
we have, as long as~$\putwo \leqslant \qvtwo$,
\[
\sigma^2_{\UU}(p^*) \leqs 2R_{\ast}^2 \cL_{\Unorm,\Vnorm}^2 + \frac{2}{n^2} R_{\ast}^2 \|\wh{X}^\top\|_{\quone\times \qvone}^2,
\]
and, as long as~$\pvone \geqs 2$,
\[
\sigma_{\VV}^2(q^*) \leqs 8n\cL_{\Unorm,\Vnorm}^2 + \frac{8}{n^2} \|\wh{X}\|_{1\times \quone}^2.
\]
\end{proposition}

\begin{proof}

Note that the dual norms to~$\|\cdot\|_{\puone\times \putwo}$ and~$\|\cdot \|_{\pvone\times \pvtwo}$ are given by~$\|\cdot\|_{\quone\times \qutwo}$ and~$\|\cdot \|_{\qvone\times \qvtwo}$ correspondingly, see, e.g.,~\cite{sra2012fast}.

$\boldsymbol{1^o}.$
For~$\E \left[\| \xi_{U}(p) \|_{\Vnorm^\ast}^2 \right]$ we have:
\[
\begin{aligned}
\E \left[\| \xi_{U}(p) \|_{\Vnorm^\ast}^2 \right] = \sum\limits_{i=1}^{2d} p_i \left\Vert \wh{X}\dfrac{e_ie_i^\top}{p_i}{U}\right\Vert_{\qvone\times \qvtwo}^2 
&= \sum_{i=1}^{2d} \frac{1}{p_i} \big\Vert \wh{X}(:,i)\cdot {U}(i,:)\big\Vert_{\qvone\times \qvtwo}^2 \\
&= \sum_{i=1}^{2d} \frac{1}{p_i} \Vert \wh{X}(:,i)\Vert_{\qvone}^2 \cdot \Vert {U}(i,:) \Vert_{\qvtwo}^2,
\end{aligned}
\]
where the last transition can be verified directly.
The right-hand side can be easily minimized on~$\Delta_{2d}$ explicitly, which results in
\[
p^{\ast}_i = \frac{\| \wh{X}(:,i) \|_{\qvone} \cdot  \| {U}(i,:) \|_{\qvtwo}}{\sum_{\imath=1}^{2d} \| \wh{X}(:,\imath) \|_{\qvone} \cdot  \| {U}(\imath,:) \|_{\qvtwo}}
\]
and 
\[
\E \left[\Vert \xi_{U}(p^{\ast}) \|_{\Vnorm^\ast}^2 \right] = \left[ \sum\limits_{i=1}^{2d} \Vert \wh{X}(:,i)\Vert_{\qvone} \cdot  \Vert {U}(i,:) \|_{\qvtwo} \right]^2.
\]
Now we can bound~$\sigma_{\UU}^2(p^*)$ via the triangle inequality:
\[
\begin{aligned}
\sigma_{\UU}^2(p^*)
&\leqs \frac{2}{n^2} \sup_{U \in \UU}\Vert \wh{X}{U}\Vert_{\Vnorm^\ast}^2 + \frac{2}{n^2}\sup_{U \in \UU}\E \left[\Vert \xi_{U}(p^{\ast})\Vert_{\Vnorm^\ast}^2 \right] \\
&= 2 R_{\ast}^2 \cL_{\Unorm,\Vnorm}^2 + \frac{2}{n^2} \sup_{U \in \UU} \left[ \sum\limits_{i=1}^{2d} \Vert \wh{X}(:,i)\Vert_{\qvone}\cdot  \Vert {U}(i,:) \Vert_{\qvtwo} \right]^2 \\
&= 2 R_{\ast}^2 \cL_{\Unorm,\Vnorm}^2 + \frac{2}{n^2} \sup_{\| U \|_{\Unorm} \le  R_{\ast}} \left[ \sum\limits_{i=1}^{2d} \Vert \wh{X}(:,i)\Vert_{\qvone}\cdot  \Vert {U}(i,:) \Vert_{\qvtwo} \right]^2 \\
&= 2R_{\ast}^2 \cL_{\Unorm,\Vnorm}^2 + \frac{2}{n^2} R_{\ast}^2\Vert \wh{X}^\top \|_{\quone \times \qvone}^2,
\end{aligned}
\]
where we used Lemma~\ref{lemma_variance_U} (see Appendix~\ref{sec:lemmas}) in the last transition.

$\boldsymbol{2^o}.$
We now deal with~$\E \left[\Vert \eta_{V,Y}(q)\Vert_{\Unorm^\ast}^2\right]$. As previously, we can explicitly compute
\[
q^{\ast}_j = \frac{\| \wh{X}(j,:) \|_{\quone}\cdot  \| V(j,:) - Y(j,:)\|_{\qutwo} }{\sum_{\jmath=1}^n \Vert \wh{X}(\jmath,:)\Vert_{\quone} \cdot  \Vert V(\jmath,:) - Y(\jmath,:)\Vert_{\qutwo}}
\]
and
\[
\E \left[\Vert \eta_{V,Y}(q^*)\Vert_{\Unorm^\ast}^2\right] = \left[\sum_{j=1}^n \Vert \wh{X}(j,:)\Vert_{\quone} \cdot  \Vert V(j,:) - Y(j,:)\Vert_{\qutwo} \right]^2.
\]
Thus, by the triangle inequality, 
\[
\begin{aligned}
\sigma_{\VV}^2(q_*) 
&\leqs \frac{2}{n^2}\sup_{(V,Y)\in\VV \times \VV}\Vert \wh{X}^\top (V-Y)\Vert_{\Unorm^\ast}^2 +  \frac{2}{n^2} \sup_{(V,Y)\in\VV \times \VV} \E \left[\Vert \eta_{V,Y}(q^*)\Vert_{\Unorm^\ast}^2\right] \\
&\leqs \frac{2}{n^2}\sup_{\| V \|_{\infty \times 1} \leqs 2} \Vert \wh{X}^\top V\Vert_{\Unorm^\ast}^2 +  \frac{2}{n^2} \sup_{\| V \|_{\infty \times 1} \leqs 2}\left[ \sum\limits_{j=1}^n \Vert \wh{X}(j,:)\Vert_{\quone}\cdot  \Vert V(j,:) \Vert_{\qutwo} \right]^2 \\
&= \frac{8}{n} \sup_{\| V \|_{2 \times 1} \leqs 1} \Vert \wh{X}^\top V\Vert_{\Unorm^\ast}^2  + \frac{8}{n^2}\|\wh{X} \|^2_{1\times\quone} \\
&\leqs \frac{8}{n} \sup_{\| V \|_{\pvone \times \pvtwo} \leqs 1} \Vert \wh{X}^\top V\Vert_{\Unorm^\ast}^2  + \frac{8}{n^2}\|\wh{X} \|^2_{1\times\quone}\\
&= 8n \cL_{\Unorm,\Vnorm}^2  + \frac{8}{n^2}\|\wh{X} \|^2_{1\times\quone}.
\end{aligned}
\]
Here in the second line we used that the Minkowski sum~$\Delta_k + (-\Delta_k)$ belongs to the~$\ell_1$-ball with radius 2 (whence~$\VV + (-\VV)$ belongs to the~$(\ell_\infty \times \ell_1)$-ball with radius~$2$);
in the third line we used Lemma~\ref{lemma_variance_V} (see Appendix~\ref{sec:lemmas}) and the relation on~$\R^{n \times k}$:
\[
\|\cdot\|_{2 \times 1} \leqs \sqrt{n} \|\cdot\|_{\infty \times 1};
\] 
lastly, we used that~$\pvone \ge 2$ and that~$\| \cdot \|_{\pvone \times \pvtwo}$ is non-increasing in~$\pvone, \pvtwo \ge 1$.
\end{proof}

\paragraph{Proof of Proposition~\ref{th:variance-partial}.} 
We instantiate Proposition~\ref{th:variance-partial-general} with~$\|\cdot\|_{\Unorm} = \|\cdot\|_{1 \times 1}$ and~$\|\cdot\|_{\Vnorm} = \|\cdot\|_{2 \times 1}$, and observe, using Proposition~\ref{th:lip}, that for~$\wh X = [X,-X] \in \R^{n \times 2d}$ it holds
\[
\cL_{\Unorm,\Vnorm} = \frac{1}{n} \|\wh X^\top\|_{\infty \times 2} = \frac{1}{n} \|X^\top\|_{\infty \times 2}
\]
and 
\[
\|\wh X\|_{1 \times \infty} = \|X\|_{1 \times \infty}.
\]
\qed

\subsection{Proof of Proposition~\ref{th:variance-full}}
\label{Theorem:variancesFull}
We have
\begin{equation*}
\label{eq_pq_star_star}
\begin{aligned}
\min_{\substack{p \in\Delta_{2d}, \\ {P \in (\Delta_k^\top)^{\otimes 2d}}}} \E\Vert \xi_{ U}(p,P)\Vert_{2\times\infty}^2 
&=\min_{\substack{p \in\Delta_{2d}, \\ {P \in (\Delta_k^\top)^{\otimes 2d}}}} \sum_{i=1}^{2d} \frac{1}{p_i} \Vert \wh{X}(:,i)\Vert_{2}^2 \cdot \left[ \sum\limits_{l=1}^k\frac{1}{P_{il}}\cdot|{U}(i,l)|^2 \right] \\
&=\min_{p \in\Delta_{2d}} \sum_{i=1}^{2d} \frac{1}{p_i} \Vert \wh{X}(:,i)\Vert_{2}^2 \cdot \|{U}(i,:)\|_1^2,
\end{aligned}
\end{equation*}
where we carried out the internal minimization explicitly, obtaining
\[
P_{il}^* = \frac{|{U}_{il}|}{\|{U}(i,:)\|_1}.
\]
Optimization in~$p$ gives:
\[
p^*_i = \frac{\Vert \wh{X}(:,i)\Vert_{2}\cdot  \Vert {U}(i,:) \Vert_{1} }{\sum\limits_{\imath=1}^{2d} \Vert \wh{X}(:,\imath)\Vert_{2}\cdot  \Vert {U}(\imath,:) \Vert_{1} }, \quad
\E\Vert \xi_U(p^*,P^*)\Vert_{2\times\infty}^2 = \sum_{i=1}^{2d} \Vert \wh{X}(:,i)\Vert_{2}\cdot  \Vert {U}(i,:) \Vert_{1}.
\]
Defining
\[
\cL_{\Unorm,\Vnorm} = \frac{1}{n} \sup_{\| U \|_{\Unorm} \leqs 1} \|X U\|_{\Vnorm^*}
\] 
and proceeding as in the proof of Proposition~\ref{th:variance-partial-general}, we get
\[
\begin{aligned}
\sigma_{\UU}^2(p^*,P^*)
&\leqs \frac{2}{n^2} \sup_{U \in \UU}\Vert \wh{X}{U}\Vert_{2 \times \infty}^2 + \frac{2}{n^2}\sup_{U \in \UU}\E \left[\Vert \xi_{U}(p^{\ast},P^*)\Vert_{2 \times \infty}^2 \right] \\
&= 2 R_{\ast}^2 \cL_{\Unorm,\Vnorm}^2 + \frac{2}{n^2} \sup\limits_{U \in \UU} \left[ \sum_{i=1}^{2d} \Vert \wh{X}(:,i)\Vert_{2}\cdot  \Vert {U}(i,:) \Vert_{1} \right]^2 \\
&= 2 R_{\ast}^2 \cL_{\Unorm,\Vnorm}^2 + \frac{2R_{\ast}^2}{n^2} \sup_{\|U\|_{1 \times 1} \leqs 1} \left[ \sum_{i=1}^{2d} \Vert \wh{X}(:,i)\Vert_{2}\cdot  \Vert {U}(i,:) \Vert_{1} \right]^2 \\
&= 2R_{\ast}^2 \cL_{\Unorm,\Vnorm}^2 + \frac{2}{n^2} R_{\ast}^2\Vert \wh{X}^\top \|_{2 \times 1}^2 \\
&= \frac{4}{n^2} R_{\ast}^2 \| X^\top \|_{2 \times 1}^2,
\end{aligned}
\]
where in the last two transitions we used Lemma \ref{lemma_variance_U} and Proposition~\ref{th:lip} (note that~$\| \wh X^\top \|_{2 \times 1} = \| X^\top \|_{2 \times 1}$). Note that the last transition requires that~$\|\cdot\|_{\Unorm}$ has~$\ell_1$-geometry in the classes -- otherwise,~Lemma~\ref{lemma_variance_U} cannot be applied. 

To obtain~$(q^*,Q^*)$ we proceed in a similar way:
\begin{equation*}
\begin{aligned}
\min_{\substack{q \in\Delta_{n}, \\ {Q \in (\Delta_k^\top)^{\otimes n}}}} \E\Vert \eta_{V,Y}(q,Q)\Vert_{\infty \times\infty}^2 
&=\min_{\substack{q \in\Delta_{n}, \\ {Q \in (\Delta_k^\top)^{\otimes n}}}} \sum_{j=1}^{n} \frac{1}{q_j} \Vert \wh{X}(j,:)\Vert_{\infty}^2 \cdot \left[ \sum\limits_{l=1}^k\frac{1}{Q_{jl}}\cdot|V(j,l) - Y(j,l)|^2 \right] \\
&=\min_{q \in\Delta_{n}} \sum\limits_{j=1}^{n} \frac{1}{q_j} \Vert \wh{X}(j,:)\Vert_{\infty}^2 \cdot \|V(j,:) - Y(j,:)\|_1^2,
\end{aligned}
\end{equation*}
which results in
\[
\begin{aligned}
q^{\ast}_j &= \frac{\| \wh{X}(j,:) \|_{\infty}\cdot  \| V(j,:) - Y(j,:)\|_{1} }{\sum_{\jmath=1}^n \Vert \wh{X}(\jmath,:)\Vert_{\infty}\cdot  \Vert V(\jmath,:) - Y(\jmath,:)\Vert_{1}} \;\;
&Q_{jl}^* &= \frac{|V_{jl} - Y_{jl}|}{\|V(j,:)-Y(j,:)\|_1}.
\end{aligned}
\]
The corresponding variance proxy can then be bounded in the same way as in the proof of Proposition~\ref{th:variance-partial-general}.
\qed

\section{Discussion of Alternative Geometries}
\label{sec:AllGeometries}

Here we consider alternative choices of the proximal geometry in mirror descent applied to the saddle-point formulation of the CCSPP~\eqref{eq:erm}, possibly with other choices of regularization than the entrywise~$\ell_1$-norm. The goal is to show that the geometry chosen in Sec.~\ref{sec:geometry-choice-main} is the only one for which we can obtain favorable accuracy guarantees for stochastic mirror descent~\eqref{StochasticUpdates}.

Given the structure of the primal and dual feasible sets, it is reasonable to consider general mixed norms of the type~\eqref{def:mixed-norms}:
\[
\Vert \cdot \Vert_{\Unorm} = \Vert \cdot \Vert_{\puone\times \putwo}, \quad
\Vert \cdot \Vert_{\Vnorm} = \Vert \cdot \Vert_{\pvone\times \pvtwo},
\]
where~$p_U^{1,2}, p_V^{1,2} \ge 1$ (in the case of~$\Vert \cdot \Vert_{\Unorm}$, we also assume the same norm for regularization).
Note that their dual norms can be easily computed: the dual norm of~$\|\cdot\|_{p^1 \times p^2}$ is~$\|\cdot\|_{q^1 \times q^2}$, where~$q^{1,2}$ are the corresponding conjugates to~$p^{1,2}$, i.e.,~$1/p^{i} + 1/q^i = 1$ (see, e.g.,~Lemma 3 in~\citet{sra2012fast}). Moreover, it makes sense to fix~$p_V^1 = 2$ for the reasons discussed in Section~\ref{sec:norms-and-potentials}. This leaves us with the obvious choices~$p_V^{2} \in \{1,2\}$,~$p_{U}^2 \in \{1,2\}$ which corresponds to the sparsity-inducing or the standard Euclidean geometry of the \textit{classes} in the dual/primal;~$p_U^{1} \in \{1,2\}$ which corresponds to the sparsity-inducing or Euclidean geometry of the \textit{features}. 
Finally, the choice~$p_U^{1} = 2$ (i.e., the Euclidean geometry in the features) can also be excluded: its combination with~$p_V^{1} = 2$ is known to lead to the large variance term in the \textit{biclass} case.\footnotemark
\footnotetext{Note that in the biclass case, our variance estimate for the partial sampling scheme (cf. Theorem~\ref{th:md-partial}) reduces to those in~\citep[Section 2.5.2.3]{optbook2}. They consider the cases of~$\ell_1/\ell_1$ and~$\ell_1/\ell_2$ geometries for the primal/dual, and omit the case of~$\ell_2/\ell_2$-geometry, in which the sampling variance ``explodes''.}
This leaves us with the possibilities
\begin{equation}
\label{eq:geometry-possibilities}
p_U^{2}, p_V^{2} \in \{1,2\} \times \{1,2\}.
\end{equation}
In all these cases, the quantity~$\cL_{\UU,\VV}$ defined in~\eqref{eq:cross-lip-bilinear} can be controlled by extending Proposition~\ref{th:lip}:
\begin{proposition}
\label{th:lip-general}
For any~$\alpha \geqslant 1$ and~$\beta \geqslant 1$ such that~$\beta \geqslant \alpha$ it holds:
\[
\cL_{\UU,\VV} := \frac{\| X \|_{1\times \alpha, \, 2\times \beta }}{n} = \frac{\| X^\top \|_{\infty\times 2}}{n}.
\]
\end{proposition}
The proof of this proposition follows the steps in the proof of Proposition~\ref{th:lip}, and is omitted.

Finally, the corresponding partial potentials could be constructed by combining the Euclidean and an entropy-type potential in a way similar to the one described in Sec.~\ref{sec:norms-and-potentials} for the dual variable; alternatively, one could use the power potential of~\citet{nesterov2013first} that results in the same rates up to a constant factor. 

Using Proposition~\ref{th:lip-general}, we can also compute the potential differences for the four remaining setups~\eqref{eq:geometry-possibilities}. The results are shown in Table~\ref{TableNonStochastic}. Up to logarithmic factors, we have equivalent results for all four geometries, with the radius~$R_{*}$ evaluated in the corresponding norm~$\|\cdot\|_{1\times 2}$ or~$\|\cdot\|_{1\times 1} = \|\cdot\|_{1} $. 
\begin{table}[ht]
\begin{center}
\begin{tabular}{|c|l|c|c|}
 \cline{3-4}
\multicolumn{2}{c|}{}&\multicolumn{2}{c|}{Norm for~$V \in \R^{n\times k}$}\\
 \cline{3-4}
\multicolumn{2}{c|}{}& $2\times 1$ & $2\times 2$ \\
 \cline{1-4}
  \multirow{7}{*}{{\vspace{1.6cm}Norm for~$U \in \R^{d\times k}$}} & $1\times 2$ & $\barr{c} \vspace{-0.4cm} \\ \Omega_{\UU} = \|U^*\|_{1\times 2}^2 \log d \\  \Omega_{\VV} = n\log k  \earr$ & $\barr{c} \vspace{-0.4cm}\\ \Omega_{\UU} = \|U^*\|_{1\times 2}^2 \log d \\   \Omega_{\VV} = n \earr$\\
  \cline{2-4}
  & $1\times 1$ & $\barr{c} \vspace{-0.4cm} \\ \Omega_{\UU} = \|U^*\|_{1}^2 \log(dk) \\   \Omega_{\VV} = n\log k  \earr$ & $\barr{c} \vspace{-0.4cm} \\ \Omega_{\UU} = \|U^*\|_{1}^2  \log(dk)  \\  \Omega_{\VV} = n \earr$\\
  \hline
\end{tabular}
  \caption{Comparison of the potential differences for the norms corresponding to~\eqref{eq:geometry-possibilities}.}
  \label{TableNonStochastic}
\end{center}
\end{table}

As a result, for the deterministic Mirror Descent (with balanced potentials) we obtain the accuracy bound (cf.~\eqref{eq:md-deterministic-result}):
\[
\Gap(\bar U^T, \bar V^T)  \leqs \frac{O(1)\cL_{\UU,\VV}\sqrt{\Omega_{\UU} \Omega_{\VV}}}{\sqrt{T}} \leqs
\frac{\wt O_{d,k}(1)R_*}{\sqrt{T}} \frac{\| X^\top \|_{\infty\times 2}}{\sqrt{n}} + \frac{\Rem}{T}.
\]
in all four cases, where~$\wt O_{d,k}(1)$ is a logarithmic factor in~$d$ and $k$, and~$R_* = \|U^*\|_{1 \times 2}$ or~$R_* = \|U^*\|_{1}$ depending on~$p_U^2 \in \{1,2\}$. In other words, the deterministic accuracy bound of Theorem~\ref{th:md-deterministic} is essentially preserved for all four geometries in~\eqref{eq:geometry-possibilities}.
On the other hand, using Proposition~\ref{th:variance-partial-general}, we obtain that in the case of~\eqref{eq:partial-sampling}, the extra part of the accuracy bound due to sampling (cf.~\eqref{eq:md-partial-result}) is also essentially preserved:
\begin{equation*}
\begin{aligned}
\E[\Gap(\bar U^T, \bar V^T)]
&\leqs  
\frac{\wt O_{d,k}(1) R_{*}}{\sqrt{T}}  \left( \frac{\|X^\top\|_{\infty\times 2}}{\sqrt{n}} + \frac{\|X \|_{1\times\infty}}{n}\right) + \frac{\Rem}{T}.
\end{aligned}
\end{equation*}

However, if we consider \textit{full sampling}, the situation changes: in the case~$p_U^2 = 2$ the variance bound that holds for~\eqref{eq:partial-sampling} is not preserved for~\eqref{eq:full-sampling}. 
This is because our argument to control the variance of the full sampling scheme always requires that~$p_U^2 \leqs 1$ (see the proof of Proposition~\ref{th:variance-full} in Appendix~\ref{Theorem:variancesFull} for details; note that for~$q_V^2$ we do not have such a restriction since the variance proxy~$\sigma_\VV^2$ is controlled on the \textit{set}~$\VV$ given by~\eqref{eq:cube-of-simplices} that has~$\ell_{\infty}\times\ell_1$-type geometry regardless of the norm~$\|\cdot\|_{\Vnorm}$.
This leaves us with the final choice between the~$\|\cdot\|_{2 \times 1}$ and~$\|\cdot\|_{2 \times 2}$ norm in the dual, as we have to use the elementwise~$\|\cdot\|_{1}$-norm in the primal.
 Both choices result in essentially the same accuracy bound (note that this choice only influences the algorithm but not the saddle-point problem itself). We have focused on the~$\|\cdot\|_{2 \times 1}$ norm because of the algorithmic considerations: with this norm, we have multiplicative updates in the case of the multiclass hinge loss, which allows for a sublinear algorithm presented in Section~\ref{sec:sublinear}.

\section{Correctness of Subroutines in Algorithm~\ref{alg:main}}
\label{sec:alg-correctness}

In this section, we recall the subroutines used in Algorithm~\ref{alg:main} -- those for performing the lazy updates and tracking the running averages -- and demonstrate their correctness.
\resetProcCounter

\begin{algorithm}[H]
\caption{\sc UpdatePrimal}
\begin{algorithmic}[1]
\REQUIRE $\wt U \in \R^{2d \times k}$,~$\alpha, \pi,\eta \in \R^{2d}$,~$l \in [k]$,~$\gamma$,~$\lambda$,~$R_*$
\STATE $L \equiv \log(2dk)$
\FOR {$i = 1$ \TO $2d$}
\STATE $\mu_i  = \pi_i - {\alpha_i \cdot \tilde{U}(i,l)} \cdot (1-e^{-2\gamma L R_* \eta_i/n})$
\ENDFOR
\STATE $M = \sum_{i=1}^{2d} \mu_i$
\STATE $\nu= \min\{e^{-2\gamma L {R}_\ast \lambda}, {R}_\ast /M\}$
\FOR {$i = 1$ \TO $2d$}
\STATE $\wt U(i,l) \gets \wt{U}(i,l) \cdot e^{-2\gamma L R_* \eta_i / n}$
\STATE $\alpha^{+}_i = \nu \cdot \alpha_i$
\STATE $\pi_i^{+} = \nu \cdot \mu_i$
\ENDFOR
\ENSURE $\wt U, \alpha^{+},\pi^{+}$ 
\end{algorithmic}
\end{algorithm}

\paragraph{Primal Updates (Procedure~\ref{proc:primal-update}).}

To demonstrate the correctness of~Procedure~\ref{proc:primal-update}, we prove the following result:
\begin{lemma}
\label{lem:primal-correct}
Suppose that at~$t$-iteration of Algorithm~\ref{alg:main}, Procedure~\ref{proc:primal-update} was fed with~$\wt U = \wt U^t, \alpha = \alpha^t, \pi = \pi^t, \eta = \eta^t, l = l^t$ for which one had
\begin{equation}
\label{eq:primal-premise}
\wt U^t(:,l) \circ \alpha^t = U^t(:,l), \quad \forall l \in [k],
\end{equation}
where~$U^t$ is the~$t$-th primal iterate of~\eqref{StochasticUpdates} equipped with~\eqref{eq:full-sampling} with the optimal sampling distributions~\eqref{eq:optimal-probabilities-full}, and~$\eta^t$ was the only non-zero column~$\eta_{V^t, Y}(:,l^t)$ of~$\eta_{V^t, Y}$. Moreover, suppose also that
\begin{equation}
\label{eq:primal-premise-distrib}
\pi^{t}(\imath) = \|U^t(\imath,:)\|_1 = \sum_{l \in [k]} {U}^t(\imath,l), \quad \imath \in [2d],
\end{equation}
were the correct norms at the~$t$-th step. Then Procedure~\ref{proc:primal-update} will output~$\wt U^{t+t}, \alpha^{t+1}, \pi^{t+1}$ such that 
\begin{equation*}
\wt{U}^{t+1}(:,l) \circ \alpha^{t+1} = U^{t+1}(:,l),\; \forall l \in [k]
\end{equation*}
and
\[
\pi^{t+1}(\imath) = \sum_{l \in [k]} {U}^{t+1}(\imath,l), \; \imath \in [2d].
\]
\end{lemma}

\begin{proof}
Recall that the matrix~$\eta^t = \eta_{V^{t},Y}$ produced in~\eqref{eq:full-sampling} has a single non-zero column~${\eta}^{t} = \eta^t(:,l^t)$, 
and according to~\eqref{StochasticUpdates}, the primal update $U^t \to U^{t+1}$ writes as (cf.~\eqref{UpdatesForU}):
\[
{U}_{il}^{t+1} = {U}_{il}^{t} \cdot {e^{-2\gamma_t R_{*} L \eta_{V^t, Y}(i,l) /n }} \cdot \min\left\{e^{-2\gamma_t R_{*}L \lambda}, {R_{*}}/{M}\right\},
\]
where
\[
L := \log(2dk), \quad  M_t := \sum_{i=1}^{2d}\sum_{l=1}^{k}{U}_{il}^{t} \cdot e^{-2\gamma_t R_{*} L \eta_{V^t, Y}(i,l)/n},
\]
and~$\eta_{V^t, Y}$ has a single non-zero column~$\eta^t = \eta^t_{V^t,Y}(:,l^t)$. 
This can be rewritten as
\begin{equation}
\label{UpdatesForUStochactic}
U_{il}^{t+1} = \left\{ 
\begin{aligned}
&U_{il}^t \cdot \nu \cdot q_i^t,  & l &= l^t, \\
&U_{il}^t \cdot \nu, & l &\ne l^t,
\end{aligned}
\right.
\end{equation}
where
\[
\begin{aligned}
q_i^t &= e^{-2\gamma_t L R_* \eta^t_i/n},\\ 
\nu &= \min\{e^{-2\gamma_t L {R}_\ast \lambda}, {R}_\ast /M\},\\
M &= \sum_{i \in [2d]} U_{i l^t}^t\cdot q_i^t +  \sum_{i \in [2d]} \; \sum_{l \in [k] \setminus \{l^t\}} U_{i l}^t.
\end{aligned}
\]
Thus,~$M$ and~$\nu$ can be expressed via~$\pi^t(i) = \sum_{l \in [k]} {U}^t(i,l)$, cf.~\eqref{eq:primal-premise-distrib}:
\begin{equation}
\label{MUdpates}
M = \sum_{i \in [2d]} \pi_i^t - \underbrace{\alpha_i^t\tilde{U}^t_{i,l^t}}_{U^t_{i,l^t}}(1-q_i^t),
\end{equation}
where we used the premise~\eqref{eq:primal-premise}. 
Now we can see that lazy updates of $\wt{U}$ can be expressed as
\begin{equation} 
\label{UpdatesAlphaU}
\begin{aligned}
\alpha_i^{t+1} = \nu \cdot \alpha_i^t, \\
\wt{U}^{t+1}_{i,l^t} = \wt{U}^{t}_{i,l^t}\cdot q_i^t,
\end{aligned}
\end{equation}
and the updates for the norms~$\pi^{t+1}$ as
\begin{equation}
\label{rhoUpdates}
\pi_i^{t+1} = \nu^t \Big[ \pi_i^t + \alpha_i^t\wt{U}^t_{i,l^t}(q_i^{t} -1) \Big]
\end{equation}
One can immediately verify that this is exactly the update produced in the call of Procedure~\ref{proc:primal-update} in~line~\ref{alg:line-eta} of Algorithm~\ref{alg:main}.
\end{proof}

\begin{algorithm}[H]
\caption{\sc UpdateDual}
\begin{algorithmic}[1]
\REQUIRE $\wt V,Y \in \R^{n \times k}$,~$\beta, \rho, \xi \in \R^{n}$,~$\ell \in [k]$,~$y \in [k]^{\otimes n}$,~$\gamma$
\STATE $\theta = e^{-2\gamma \log(k)}$
\FOR{$j = 1$ \TO $n$}
\STATE $\omega_j = e^{2 \gamma \log(k) \xi_j}$ 
\STATE $\veps_j = e^{-2\gamma \log(k)Y(j,\ell)}$
\STATE $\chi_j = 1-\beta_j \cdot \wt{V}(j,\ell) \cdot (1-\omega_{j} \cdot \veps_j)$
\IF[not the actual class of~$j$ drawn]{$\ell \ne y_j$}
\STATE $\chi_j \gets \chi_j - \beta_j \cdot \wt{V}(j,y_j) \cdot (1 - \theta)$
\ENDIF
\label{alg:line-chi}
\STATE $\beta_j^+ = \beta_j / \chi_j$
\label{alg:line-newbeta}
\STATE $\wt{V}(j,\ell) \gets \wt{V}(j,\ell) \cdot \omega_j \cdot \veps_j$
\label{alg:line-newVl}
\STATE $\wt{V}(j,y_j) \gets \wt{V}(j,y_j) \cdot \omega_j \cdot \theta$
\label{alg:line-newVyj}
\STATE $\rho^{+}_{j} = 2 - 2\beta_j^{+}\cdot \wt{V}(j,y_j)$
\label{alg:line-newrho}
\ENDFOR
\ENSURE $\wt V, \beta^+, \rho^+$ 
\end{algorithmic}

\end{algorithm}

\paragraph{Dual Updates (Procedure~\ref{proc:dual-update}).}
To demonstrate the correctness of~Procedure~\ref{proc:dual-update}, we prove the following lemma.

\begin{lemma}
\label{lem:dual-correct}

Suppose that at~$t$-iteration of Algorithm~\ref{alg:main}, Procedure~\ref{proc:dual-update} was fed with~$\wt V = \wt V^t, \beta = \beta^t, \rho = \rho^t, \xi = \xi^t, \ell = \ell^t$, for which one had
\begin{equation}
\label{eq:dual-premise}
\wt V^t(:,l) \circ \beta^t = V^t(:,l), \quad \forall l \in [k],
\end{equation}
where~$V^t$ is the~$t$-th dual iterate of~\eqref{StochasticUpdates} equipped with~\eqref{eq:full-sampling} with the optimal sampling distributions~\eqref{eq:optimal-probabilities-full}, and~$\xi^t$ was the only non-zero column~$\xi_{U^t}(:,\ell^t)$ of~$\xi_{U^t}$. Moreover, suppose also that
\begin{equation}
\label{eq:dual-premise-distrib}
\rho^t(\jmath) = \|V^t(\jmath,:) - Y(\jmath,:)\|_1, \quad \jmath \in [n] 
\end{equation}
were the correct norms at the~$t$-th step. Then Procedure~\ref{proc:dual-update} will output~$\wt V^{t+t}, \beta^{t+1}, \rho^{t+1}$ such that 
\begin{equation*}
\wt V^{t+1}(:,l) \circ \beta^{t+1} = V^{t+1}(:,l), \; \forall{l \in [k]} 
\end{equation*}
and
\[
\rho^{t+1}(\jmath) = \|V^{t+1}(\jmath,:) - Y(\jmath,:)\|_1, \quad \jmath \in [n].
\]
\end{lemma}

\begin{proof}

Recall that the random matrix~$\xi_{U^{t}}$ has a single non-zero column~$\xi^t := \xi_{U^{t}}(:,\ell^t)$, 
and according to~\eqref{StochasticUpdates}, the update~$V^{t} \to V^{t+1}$ writes as~(cf.~\eqref{UpdatesForVSVM}):
\begin{equation}
\label{UpdatesForVStochactic}
V_{jl}^{t+1} = V_{jl}^t \cdot \frac{\exp[2\gamma_t\log(k) \cdot (\xi_{U^t}(j,l) - Y(j,l))]}{\sum\limits_{\ell=1}^k V_{j\ell}^t\cdot \exp[2\gamma_t\log(k) \cdot (\xi_{U^t}(j,\ell) - Y(j,\ell))]}.
\end{equation}
Note that all elements of the matrix~$\xi_{U^t} - Y$ in each row~$j$ 
have value~$1$, except for at most two elements in the columns~$\ell^t$ and~$y_j$, where~$y_j$ is the actual label of the $j$-th training example, that is, the only~$l \in [k]$  for which~$Y(j,l) = 1$. Recall also that~$\sum_{\ell \in [k]} V^t(j,\ell) = 1$ for any~$j \in [n]$. 
Thus, introducing
\[
\omega_j = e^{2 \gamma_t \log(k) \xi_j^t}, \quad \veps_j = e^{-2\gamma_t \log(k)Y(j,l)}
\]
as defined in Procedure~\ref{proc:dual-update},
we can express the denominator in~\eqref{UpdatesForVStochactic} as
\begin{equation}
\label{xiUpdates}
 \chi_j =  
 \left\{ 
\begin{array}{ll}
1 - \underbrace{\beta^t_j \cdot \wt{V}^t_{j,\ell^t}}_{{V}^t_{j,\ell^t}} \cdot (1-\omega_{j} \cdot \veps_j), \quad &\text{if} \;  \ell^t = y_j, \\
1 - \underbrace{\beta^t_j \cdot \wt{V}^t_{j,\ell^t}}_{{V}^t_{j,\ell^t}} \cdot (1-\omega_{j} \cdot \veps_j) - \underbrace{\beta^t_j \cdot \wt{V}^t_{j,y_j}}_{{V}^t_{j,y_j}} \cdot (1 - e^{-2\gamma_t \log(k)}), \quad &\text{if} \;  \ell^t \ne y_j,
\end{array} 
\right.
\end{equation}
where we used the premise~\eqref{eq:dual-premise}. One can verify that this corresponds to the value of~$\chi_j$ produced by line~\ref{alg:line-chi} of Procedure~\ref{proc:dual-update}. Then, examining the numerator in~\eqref{UpdatesForVStochactic}, we can verify that lines~\ref{alg:line-newbeta}--\ref{alg:line-newVyj} guarantee that
\[
\wt V^{t+1}_{j,l} \cdot \beta^{t+1}_j = V^{t+1}_{j,l}, \; \forall{l \in [k]} 
\]
holds for the updated values. 
To verify the second invariant, we combining this result with the premise~\eqref{eq:dual-premise-distrib}. This gives
\begin{equation}
\label{UpdatesOmega}
\rho^{t+1}_{j} = 2 - V^{t+1}_{j,y_j} =  2 - 2 \beta_j^{t+1}\cdot \wt{V}^{t+1}_{j,y_j},
\end{equation}
which indeed corresponds to the update in line~\ref{alg:line-newrho} of the procedure.
\end{proof}

\paragraph{Correctness of Tracking the Cumulative Sums.}
We only consider the primal variables (Procedure~\ref{proc:primal-track} and line~\ref{alg:line-post-prim} of Algorithm~\ref{alg:main}); the complimentary case can be treated analogously. 
Note that due to the previous two lemmas, at any iteration~$t$ of Algorithm~\ref{alg:main} Procedure~\ref{proc:primal-track} is fed with~$l = l^t$,~$\wt U = \wt U^t$,~$\alpha = \alpha^t$ for which 
it holds~$\wt U^t \alpha^t = U^t$. 
Now, assume that all previous input values~$\Acur^\tau, \tau \le t,$ of variable~$\Acur$, and the current inputs~$\Aprev^t, \Usum^t$ of variables~$\Aprev, \Usum$, satisfy the following:
\begin{align}
\Acur^\tau &= \sum_{s = 0}^{\tau-1} \alpha^{s}, \quad \forall \tau \le t, \label{eq:Acur-inv}\\
\Aprev^t(i,l)  &= A^{\tau^t(i,l)}_{i}, \label{eq:Aprev-inv}\\
\Usum^t(i,l) &= \sum_{s = 0}^{\tau^t(i,l)} U^s(i,l), \label{eq:Usum-inv}
\end{align}
where~$0 \leqs \tau^t(i,l) \leqs t-1$ is the latest moment~$s$, \textbf{strictly before}~$t$, when the sampled~$l^{s} \in [k]$ coincided with the given~$l$:
\begin{equation}
\tau^t(i,l) = \argmax_{s \leqs t-1} \{s: l^{s} = l\}. \label{eq:Closing-inv}
\end{equation}
Let us show that this invariant will be preserved aftet the call of Procedure~\ref{proc:primal-track} -- in other words, that~\eqref{eq:Acur-inv}--\eqref{eq:Closing-inv} hold for~$t + 1$, i.e., for the ouput values~$\Usum^{t+1}, \Aprev^{t+1}, \Acur^{t+1}$ (note that the variables~$\Usum, \Aprev, \Acur^{t+1}$ only changed within Procedure~\ref{proc:primal-track}, so their output values are also the input values at the next iteration).

\begin{proof}
Indeed, it is clear that~\eqref{eq:Acur-inv} will be preserved (cf.~line~\ref{alg:line-Acur-update} of Procedure~\ref{proc:primal-track}). 
To verify~\eqref{eq:Aprev-inv}, note that~$\Aprev(i,l)$ only gets updated when~$l = l^t$ (cf.~line~\ref{alg:line-Aprev-update}), and in this case we will have~$\tau^{t+1}(i,l) = t$, and otherwise~$\tau^{t+1}(i,l) = \tau^{t}(i,l)$, cf.~\eqref{eq:Closing-inv}. 

Thus, it only remains to verify the validity of~\eqref{eq:Usum-inv} after the update. To this end, note that by~\eqref{eq:Closing-inv} we know that the value~$\wt{U}^s(i,l)$ of the variable~$\wt U(i,l)$ remained constant for~$\tau^t(i,l) \le s < t$, and it will not change after the call at $t$-th iteration unless~$l^t = l$, that is, unless~$\tau^{t+1}(i,l) = t$. This is exactly when line~\eqref{alg:line-Usum-update} is invoked, and it ensures~\eqref{eq:Usum-inv} for~$t+1$.
\end{proof}

Finally, invoking~\eqref{eq:Acur-inv}--\eqref{eq:Closing-inv} at~$t = T$, we see that line~\ref{alg:line-post-prim} results in the correct final value~$\sum_{t= 0}^{T} U^t$ of the cumulative sum~$\Usum$. 
Thus, the correctness of Algorithm~\ref{alg:main} is verified.

\section{Additional Remarks on Algorithm~\ref{alg:main}}
\label{sec:alg-details}

\paragraph{Removing the~$O(dk)$ Complexity Term.}
In fact, the extra term~$O(dk)$ in the runtime and memory complexities of Algorithm~\ref{alg:main} can be easily avoided. 
To see this, recall that when solving the simplex-constrained CCSPP~\eqref{SaddleUhat}, we are foremost interested in solving the~$\ell_1$-constrained CCSPP~\eqref{MainEquationConstrained}, and an~$\veps$-accurate solution~$U = [U_1; U_2] \in \R^{2d \times k}$ to~\eqref{SaddleUhat} yields an~$\veps$-accurate solution~$\wh U = U_1 - U_2 \in \R^{d \times k}$ to~\eqref{MainEquationConstrained}. 
Recall that we initialize Algorithm~\ref{alg:main} with~$\wt U^0 = \ind_{2d \times k}$ and~$\alpha^0 = \ind_{2d}$, which corresponds to~$\wh U^0 = 0_{2d \times k}$. 
Moreover, at any iteration we change a single entry of~$\wt U^t$, and scale the whole scaling vector~$\alpha^0 = \ind_{2d}$ by a constant (in fact, all entries of~$\alpha^t$ are always equal to each other; we omitted this fact in the main text to simplify the presentation, since the entries of~$\beta$ generally have different values). Hence, the final candidate solution~$\wh U^{T} = [\bar U^T_1 - \bar U^T_2]$ to the~$\ell_1$-constrained problem will actually have at most~$O(dT)$ non-zero entries that correspond to the entries of~$\wt U$ changed in the course of the algorithm. To exploit this, we can modify Algorithm~\ref{alg:main} as follows: 

\begin{itemize}
\item
Instead of explicitly initializing and storing the whole matrices~$\wt U$,~$\Usum$, and~$\Aprev$, we can hard-code the ``default'' value~$\wt U(i,l) = 1$ (cf.~line~\ref{alg:line-U-init} of Algorithm~\ref{alg:main}), and use a bit mask to flag the entries~$\wt U(i,l)$ that have already been changed at least once. This mask can be stored as a list~$\List$ of pairs $(i,l)$, i.e. in a sparse form.
\item 
When post-processing the cumulative sum~$\Usum$ (see line~\eqref{alg:line-post-prim} of Algorithm~\ref{alg:main}), instead of post-processing all entries of~$\Usum$, we can only process those in the list~$\List$, and ignore the remaining ones, since the corresponding to them entries in~$\wh U^T$ (a candidate solution to~\eqref{MainEquationConstrained}) will have zero values. We can then directly output~$\wh U^T$ in a sparse form.
\end{itemize}

It is clear that such modification of Algorithm~\ref{alg:main} results in the replacement of the~$O(dk)$ term in runtime complexity with~$O(dT)$ (which is always an improvement since~$O(d)$ a.o.'s are done anyway in each iteration); 
moreover, the memory complexity changes from~$O(dn + nk + dk)$ to~$O(dn + nk + d \min(T,k))$.

\paragraph{Infeasibility of the Noisy Dual Iterates.}
Note that when we generate an estimate of the primal gradient~$X^T(V^t-Y)$ according to~\eqref{eq:full-sampling} or~\eqref{eq:partial-sampling}, we also obtain an unbiased estimate of the dual iterate~$V^t$, and vice versa. In the setup with vector variables,~\citet{optbook2} propose to average such noisy iterates instead of the acutal iterates~$(U^t, V^t)$ as we do in~\eqref{StochasticUpdates}. Averaging of the noisy iterates is easier to implement since they are sparse (one does not need to track the cumulative sums), and one could show similar guarantees for the primal accuracy of their running averages. However, in the case of the dual variable its noisy counterpart is infeasible~\citep[Sec.~2.5.1]{optbook2}; as a result, one loses the guarantee for the duality gap. Hence, we prefer to track the averages of the actual iterates of~\eqref{StochasticUpdates} as we do in Algorithm~\ref{alg:main}.

\end{document}